\documentclass[journal]{IEEEtran}

\usepackage{adjustbox}
\usepackage{algorithm}
\usepackage{algorithmic}
\usepackage{amsmath,amssymb,amsfonts,amsthm}
\usepackage{array}
\usepackage{blkarray}
\usepackage{bm}
\usepackage{booktabs}
\usepackage{cite}
\usepackage{diagbox}
\usepackage{dsfont}
\usepackage[inline]{enumitem}
\usepackage{graphicx}
\usepackage{hyperref}
\usepackage{makecell}
\usepackage{mathtools}
\usepackage{multirow}
\usepackage{nth}
\usepackage{siunitx}
\usepackage{soul}
\usepackage{stmaryrd}
\usepackage{subcaption}
\usepackage{tabularx}
\usepackage{textcomp}
\usepackage{xcolor}
\usepackage{ulem}

\usepackage[mode=buildnew]{standalone}
\usepackage{pgfplots}
\usepackage{tikz}
\usetikzlibrary{patterns}

\pgfplotsset{compat=newest}
\usetikzlibrary{plotmarks}
\usetikzlibrary{arrows.meta}
\usepgfplotslibrary{patchplots}
\usepackage{grffile}

\newcommand{\norm}[1]{\left\lVert#1\right\rVert}
\newcommand{\abs}[1]{\left\lvert#1\right\rvert}
\newcommand{\nint}[1]{\llbracket#1\rrbracket}

\def\*#1{\mathbf{#1}}
\newcommand{\bfemph}[1]{\textbf{#1}}

\def\defequal{\stackrel{\mbox{\footnotesize def}}{=}}

\newcommand{\bHtilde}{\boldsymbol{\tilde{\mathbf{H}}}}
\newcommand{\bVtilde}{\boldsymbol{\tilde{\mathbf{V}}}}
\newcommand{\bWtilde}{\boldsymbol{\tilde{\mathbf{W}}}}
\newcommand{\bwtilde}{\boldsymbol{\tilde{\mathbf{w}}}}
\newcommand{\bXtilde}{\boldsymbol{\tilde{\mathbf{X}}}}

\newcommand{\mJ}{\mathcal{J}}

\newcommand{\EE}{\mathbb{E}}
\newcommand{\FF}{\mathbb{F}}
\newcommand{\NN}{\mathbb{N}}
\newcommand{\RR}{\mathbb{R}}

\DeclareMathOperator{\Diag}{Diag}

\DeclareMathOperator{\Lagr}{\mathcal{L}}

\newtheorem{lem}{Lemma}
\newtheorem{theorem}{Theorem}

\SetSymbolFont{stmry}{bold}{U}{stmry}{m}{n}

\pgfplotsset{
  log ticks with fixed point,
}

\graphicspath{{img/}}

\begin{document}


\title{Majorization-minimization for Sparse Nonnegative Matrix Factorization
  with the $\beta$-divergence}

\author{Arthur~Marmin,
        Jos{\'e}~Henrique~de~Morais~Goulart,
        and~C{\'e}dric~F{\'e}votte,~\IEEEmembership{Fellow,~IEEE}%
\thanks{This work is supported by the European Research Council
  (ERC FACTORY-CoG-6681839), the French Agence Nationale de la Recherche (ANITI,
  ANR-19-P3IA-0004) and the National Research Foundation, Prime Minister’s Office,
  Singapore under its Campus for Research Excellence and Technological
  Enterprise (CREATE) programme.}%
\thanks{A.~Marmin is with Aix-Marseille Universit{\'e}, CNRS, I2M, UMR 7373,
  Marseille, France (email: arthur.marmin@univ-amu.fr).}%
\thanks{J.~H.~de M.~Goulart is with IRIT, Universit{\'e} de Toulouse,
  Toulouse INP, Toulouse, France (e-mail: henrique.goulart@irit.fr).}%
\thanks{C.~F{\'e}votte is with IRIT, Universit{\'e} de Toulouse, CNRS, Toulouse,
  France (email: cedric.fevotte@irit.fr).}}

\maketitle


\begin{abstract}
  This article introduces new multiplicative updates for nonnegative matrix
  factorization with the $\beta$-divergence and sparse regularization of one of the
  two factors (say, the activation matrix).
  It is well known that the norm of the other factor (the dictionary matrix)
  needs to be controlled in order to avoid an ill-posed formulation.
  Standard practice consists in constraining the columns of the dictionary to
  have unit norm, which leads to a nontrivial optimization problem.
  Our approach leverages a reparametrization of the original problem into the
  optimization of an equivalent scale-invariant objective function.
  From there, we derive block-descent majorization-minimization algorithms that
  result in simple multiplicative updates for either $\ell_{1}$-regularization or the
  more ``aggressive'' log-regularization.
  In contrast with other state-of-the-art methods, our algorithms are universal
  in the sense that they can be applied to any $\beta$-divergence (i.e., any value
  of $\beta$) and that they come with convergence guarantees.
  We report numerical comparisons with existing heuristic and Lagrangian methods
  using various datasets: face images, an audio spectrogram, hyperspectral data,
  and song play counts.
  We show that our methods obtain solutions of similar quality at convergence
  (similar objective values) but with significantly reduced CPU times.
\end{abstract}




\section{Introduction} \label{sec:intro}

Nonnegative matrix factorization (NMF) consists in decomposing a data matrix
$\*V$ with nonnegative entries into the products $\*W\*H$ of two nonnegative
matrices~\cite{Paatero_P_1994_j-environ_positive_mfnfmoueedv,
  Lee_D_1999_j-nature_learning_ponmf}.
When the data samples are arranged in the columns of $\*V$, the first factor
$\*W$ can be interpreted as a dictionary of basis vectors (or atoms).
The second factor $\*H$, termed activation matrix, contains the expansion
coefficients of each data sample onto the dictionary.
NMF has found many applications such as feature extraction in image processing
and text mining~\cite{Lee_D_1999_j-nature_learning_ponmf}, audio source
separation~\cite{Smaragdis_P_2014_j-ieee-sig-proc-mag_static_dssunmfuv}, blind
unmixing in hyperspectral imaging~\cite{Berry_W_2007_j-comput-stat-data-anal_algorithms_aanmf,
  BioucasDias_J_2012_j-ieee-j-sel-top-appl-earth-obs-rem-sens_hyperspectral_uogssrba},
and user recommendation~\cite{Hu_Y_2008_p-ieee-icdm_collaborative_fifd}.
For a thorough presentation of NMF\@ and its applications,
see~\cite{Cichoki_A_2009_book_nonnegative_mtf,
  Fu_X_2019_j-ieee-sig-proc-mag_nonnegative_mfsdaiaa,
  Gillis_N_2020_book_nonnegative_mf}.

NMF is usually cast as the minimization of a well-chosen measure of fit between
$\*V$ and $\*W \*H$.
A widespread choice for the measure of fit is the $\beta$-divergence, a family of
divergences parametrized by a single shape parameter $\beta \in \RR$.
This family notably includes the squared Frobenius norm (quadratic loss) as well
as the Kullback-Leibler (KL) and Itakura-Saito (IS)
divergences~\cite{Cichocki_A_2011_j-entropy_generalized_abdtarnmf,
  Fevotte_C_2011_j-neural-comput_algorithms_nmfbd}.
NMF is well-known to favor part-based representations that \textit{de facto}
produce a sparse representation of the input data (because of the sparsity of
either $\*W$ or $\*H$)~\cite{Lee_D_1999_j-nature_learning_ponmf}.
This is a consequence of the nonnegativity constraints that produce zeros on the
border of the admissible domain of $\*W$ and $\*H$.
However, it is sometimes desirable to accentuate or control the sparsity of the
factors by regularizing NMF with specific sparsity-promoting terms.
This can improve the interpretability or suitability of the resulting
representation as illustrated in the seminal work of
Hoyer~\cite{Hoyer_P_2002_p-nnsp_nonnegative_sc,
  Hoyer_P_2004_j-mach-learn-res_nonnegative_mfsc}.

Sparse regularization using penalty terms on the factors is the most common
approach to induce sparsity.
A classic penalty term is the $\ell_{1}$ norm (the sum of the entries of the
nonnegative factor), used for example in~\cite{Hoyer_P_2002_p-nnsp_nonnegative_sc,
  Eggert_J_2004_p-ijcnn_sparse_cnmf,
  Kim_H_2007_j-bioinformatics_sparse_nmfanclsmda,
  Cichoki_A_2009_j-ieice-tfeccs_fast_lalsnmtf,
  Mairal_J_2010_j-mach-learn-res_online_lmfsc,
  Guan_N_2012_j-ieee-trans-sig-proc_nenmf_ogmnmfu,
  Zhao_R_2018_j-ieee-trans-sig-proc_unified_camuarnmf}.
Other $\ell_{p}$ norms such as the $\ell_{1/2}$ norm have also been considered,
e.g.,~\cite{Qian_Y_2010_p-dicta_l1/2_lscnmffhu,
  Sigurdsson_J_2014_j-ieee-trans-geo-rem-sens_hyperspectral_ur}.
Other works have considered log-regularization (i.e., penalizing the sum of the
logarithms of the entries of the factor) which leads to more ``aggressive''
sparsity, e.g.,~\cite{Lefevre_A_2011_p-icassp_itakura_isnmfgs,
  Tan_V_2013_j-ieee-trans-pami_automatic_rdnmfd,
  Peng_C_2022_j-knowledge-bases-sys_log_snmfdr}.
Sparse regularization with information measures is also considered
in~\cite{Shashanka_M_2007_p-nips_sparse_olvdcd,
  Joder_C_2013_p-icassp_comparative_sspnmfssbn}.
Another approach to induce sparsity consists in applying hard constraints to the
factors (rather than mere penalization), using $\ell_{0}$
constraints~\cite{Peharz_R_2012_j-neurocomputing_sparse_nmfwc,
  Bolte_J_2013_j-math-prog_proximal_almnnp} or using the sparseness measure
introduced in~\cite{Hoyer_P_2004_j-mach-learn-res_nonnegative_mfsc}.
Regularization of NMF with group-sparsity has also been a very active topic,
see, e.g., early references~\cite{Lefevre_A_2011_p-icassp_itakura_isnmfgs,
  Kim_J_2012_p-siam-icdm_group_snmf,
  Tan_V_2013_j-ieee-trans-pami_automatic_rdnmfd}.

In this paper, we assume without loss of generality that the sparse
regularization is applied to $\*H$ (our results apply equally as well to $\*W$
by transposing $\*V$ and exchanging of the roles of $\*W$ and $\*H$).
NMF with sparse regularization of $\*H$ requires controlling the norm of $\*W$.
Indeed, the measure of fit only depends on the product of $\*W$ and $\*H$ while 
the regularization term solely depends on $\*H$: it is then possible to
arbitrarily decrease the overall objective function by decreasing the scale of
$\*H$ and increasing the scale of $\*W$ (this will be made more precise in
Section~\ref{ssec:well_pb}).
{
There are two main approaches to control the norm of $\*W$.
The first and most common approach, used in many of the references above,
e.g.,~\cite{Hoyer_P_2002_p-nnsp_nonnegative_sc,
  Hoyer_P_2004_j-mach-learn-res_nonnegative_mfsc,
  Eggert_J_2004_p-ijcnn_sparse_cnmf,  
  Cichoki_A_2009_j-ieice-tfeccs_fast_lalsnmtf,
  Mairal_J_2010_j-mach-learn-res_online_lmfsc}, consists in imposing that the norms
(either $\ell_{1}$ or $\ell_{2}$) of the individual columns of $\*W$ are less or equal to one.\footnote{{As a matter of fact, it is easy to show that imposing the
norms of the columns to be less than $1$ actually produces solutions with
saturated norm equal to $1$~\cite{Mairal_J_2010_j-mach-learn-res_online_lmfsc}.}}
The second approach consists of merely penalizing the norm of $\*W$ by adding
a supplementary regularization term to the objective function,
e.g.,~\cite{Kim_H_2007_j-bioinformatics_sparse_nmfanclsmda}.
The first approach fits well with the dictionary learning view of matrix
factorization.
It means that the atoms of the dictionary (the columns of $\*W$) are normalized
and only convey shape information.
All scaling information is relegated to the coefficients of the activation
matrix $\*H$.
It is also consistent with traditional NMF practice (i.e., NMF without sparsity
constraints) where the columns of $\*W$ are often renormalized (together with
the rows of $\*H$) after every iteration in order to solve the scale ambiguity
between $\*W$ and $\*H$.
The second approach (penalization of the norm of $\*W$) leads to more simple
optimization problems but is less interpretable in the perspective of dictionary
learning as it is bound to return columns of unequal norms.
In this paper we consider the first and most common approach.
Next, we review methods that have been proposed to enforce the unit-norm
constraint on the columns of $\*W$ in the context of sparse NMF with the
$\beta$-divergence.
}


\subsection{State of the art} A first strategy, employed
in~\cite{Hoyer_P_2002_p-nnsp_nonnegative_sc,
  Mairal_J_2010_j-mach-learn-res_online_lmfsc}, consists in using projected
gradient descent for the update of $\*W$.
This procedure works well with the quadratic loss function and unit $\ell_{2}$ norm
constraint that is used in those papers.

A second strategy consists in reparametrizing $\*W$ as
$\*W \Diag^{-1}(\norm{\*w_{1}}, \ldots, \norm{\*w_{K}})$, where $\*w_{k}$
denotes the $k$-th column of $\*W$, and optimizing over the new rescaled
variable.
This approach was proposed in~\cite{Eggert_J_2004_p-ijcnn_sparse_cnmf} for NMF
with the quadratic loss and was extended to NMF with the $\beta$-divergence
(referred to as $\beta$-NMF in the following) in~\cite{LeRoux_J_2015_TR_sparse_nmfhwd}.
The proposed update for $\*W$ is heuristic (it will be presented in
Section~\ref{ssec:heuristic}); while successful in practice, it lacks a proof of
convergence (in particular, it does not ensure non-increasingness of the
objective function as it will be illustrated in Section~\ref{sssec:descent-prop}).

A third strategy consists in following a Lagrangian approach and minimizing an
augmented objective function that includes the desired constraints on the
columns of $\*W$.
This is the approach pursued
in~\cite{Leplat_V_2021_j-siam-j-matrix-anal-appl_multiplicative_unmfbddec} for
$\beta$-NMF\@ and described in Section~\ref{ssec:lagr}.
Unfortunately, practical updates are only obtained for $\beta \le 1$ and specific
values $\beta \in \{ \frac{5}{4}, \frac{4}{3}, \frac{3}{2}, 2 \}$.
This excludes most of the interval $\beta \in ]1,2[$ which is of applicative interest.
For admissible values of $\beta$, this method offers theoretical guarantees and good
experimental performance.
However, the update of Lagrangian multipliers requires a numerical procedure
that can be costly.
In the special case $\beta=1$, the Lagrangian multipliers have a closed-form
expression that simplifies the updates
(see, e.g.,~\cite{Filstroff_L_2021_j-ieee-trans-sig-proc_comparative_sgmctnmf}).

Finally, a last strategy consists in rewriting sparse NMF as the optimization
of an equivalent scale-invariant objective function.
In this approach, detailed in Section~\ref{sec:snmf}, the rows of $\*H$ are
multiplied by the norms of the columns of $\*W$ and the new objective function
can be optimized without norm constraints.
In this scheme, the columns of $\*W$ can be normalized at the end of the
optimization (and the rows of $\*H$ rescaled accordingly).
This approach was applied to NMF with the IS divergence and log-regularization
in~\cite{Lefevre_A_2011_p-icassp_itakura_isnmfgs} and to NMF with
the KL divergence and a Markov regularization of the rows
of $\*H$ in~\cite{Essid_S_2013_j-ieee-trans-multimed_smooth_nmfuads}.
In these two cases, a block-descent Majorization-Minimization (MM) algorithm
was proposed.
The approach is well-posed and does not rely on any heuristic.
The MM algorithm results in simple multiplicative updates that ensure
non-increasingness of the objective function.


\subsection{Contributions} In this paper, we generalize the approach
of~\cite{Lefevre_A_2011_p-icassp_itakura_isnmfgs,
  Essid_S_2013_j-ieee-trans-multimed_smooth_nmfuads} to NMF with every possible
$\beta$-divergence, i.e., for all $\beta \in \RR$ and not merely $\beta=0$ and $\beta=1$.
More precisely, we first design a universal block-descent MM algorithm for
$\beta$-NMF with $\ell_{1}$-regularization of $\*H$, and unit $\ell_{1}$ norm constraint on the
columns of $\*W$.
This algorithm extends~\cite{Essid_S_2013_j-ieee-trans-multimed_smooth_nmfuads},
that was specifically designed for the KL divergence and a different
regularization term.
Then, we design another universal block-descent MM algorithm for $\beta$-NMF with
log-regularization of $\*H$, and unit $\ell_{1}$ norm constraint on the columns of
$\*W$.
The algorithm extends~\cite{Lefevre_A_2011_p-icassp_itakura_isnmfgs} that was
aimed at the IS divergence solely.
In both cases, the block-descent MM approach leads to alternating multiplicative
updates that are free of tuning parameters.
They are easy to implement and enjoy linear complexity per iteration.
By design, the MM framework ensures the non-increasingness and thus the
convergence of the objective function.
We further show the convergence of the iterates to the set of stationary points
of the problem using the theoretical framework
of~\cite{Zhao_R_2018_j-ieee-trans-sig-proc_unified_camuarnmf}. 

Then we demonstrate the practical advantages of our method with extensive
simulations using datasets arising from various applications: face images,
audio spectrogram, hyperspectral data and song play-counts.
We compare our MM algorithm for $\ell_{1}$-regularized $\beta$-NMF with the heuristic
presented in~\cite{LeRoux_J_2015_TR_sparse_nmfhwd} and also with the Lagrangian
method from~\cite{Leplat_V_2021_j-siam-j-matrix-anal-appl_multiplicative_unmfbddec}.
Additionally, we adapt the heuristic of~\cite{LeRoux_J_2015_TR_sparse_nmfhwd}
for $\beta$-NMF with log-regularization and compare it with our MM algorithm.
In all cases, we show that our MM algorithms obtain solutions whose quality is
similar to that of existing algorithms at convergence (similar objective values)
but often with significantly reduced CPU times.
Moreover, our algorithms overcome some of the limitations of these other
approaches, namely that the heuristic~\cite{LeRoux_J_2015_TR_sparse_nmfhwd} has
no convergence guarantees and that the Lagrangian
approach~\cite{Leplat_V_2021_j-siam-j-matrix-anal-appl_multiplicative_unmfbddec}
cannot be applied to any value of $\beta$.


\subsection{Outline} The rest of this article is organized as follows.
Section~\ref{sec:prelem} introduces $\beta$-NMF with $\ell_{1}$-regularization of the
activation matrix $\*H$.
It explains the necessity of controlling the norm of $\*W$ to formulate a
well-posed optimization problem.
Section~\ref{sec:soa} details the state of the art for the latter problem, and
more precisely the heuristic method~\cite{LeRoux_J_2015_TR_sparse_nmfhwd} and the
Lagrangian method~\cite{Leplat_V_2021_j-siam-j-matrix-anal-appl_multiplicative_unmfbddec}.
Section~\ref{sec:snmf} presents our universal block-descent MM algorithm for
$\beta$-NMF with $\ell_{1}$-regularization.
The derivations lead to multiplicative updates with convergence guarantees.
Section~\ref{sec:log-snmf} extends the methodology of Section~\ref{sec:snmf} to
$\beta$-NMF with log-regularization of $\*H$.
Experimental results are presented in Section~\ref{sec:simul} and
Section~\ref{sec:concl} concludes.


\subsection{Notation} The set $\NN$ denotes the set of natural numbers while
$\nint{1,N}$ denotes its subset containing natural numbers from $1$ to $N$.
The set $\RR_{+}$ denotes the set of nonnegative real numbers.
Bold upper case letters denote matrices, bold lower case letters denote vectors,
and lower case letters denote scalars.
The notation ${[\*M]}_{ij}$ and $m_{ij}$ both stand for the element of $\*M$
located at the $i^{\text{th}}$ row and the $j^{\text{th}}$ column.
The operators $\odot$ and $/$, and $^{.\alpha}$ applied to matrices denote the entry-wise
multiplication, division and power $\alpha$, respectively.
For a matrix $\*M$, the notation $\*M \ge 0$ denotes entry-wise nonnegativity.
The vector $\*1_{N}$ and the matrix $\*1_{F \times N}$ are the vector of
dimension $N$ and the matrix of dimension $F \times N$ composed solely of $1$
respectively.


\section{NMF with $\beta$-divergence and $\ell_{1}$ regularization}
\label{sec:prelem}

\subsection{Objective}
\label{sec:obj}

Our goal is to factorize an $F \times N$ nonnegative data matrix $\*V$ into the
product $\*W\*H$ of two nonnegative factor matrices of dimensions $F \times K$ and
$K \times N$ respectively.
The inner rank $K$ is assumed to be a fixed parameter of the problem.
Placing a $\ell_{1}$-regularization term on $\*H$, we aim at solving the following
problem
\begin{equation}
  \label{eq:pb1}
  \begin{aligned}
    &\min_{\*W, \*H \ge 0 } & & \mJ(\*W,\*H) \defequal D_{\beta}(\*V \mid \*W\*H) + \alpha\norm{\*H}_{1} \\
    &\textrm{s.t.} & & (\forall k \in \nint{1,K}) \ \norm{\*w_{k}}_{1}=1 \, ,
  \end{aligned}
\end{equation}
where $\norm{\*H}_{1} = \sum_{k,n} \abs{h_{kn}} = \sum_{k,n} h_{kn}$ and $\alpha$ is a
nonnegative hyperparameter that governs the degree of sparsity of $\*H$.
The data-fitting term $D_{\beta}$ is defined as
\begin{equation*}
  D_{\beta}(\*V \mid \*W\*H) = \sum_{f=1}^{F}\sum_{n=1}^{N} d_\beta(v_{fn}\mid{[\*W\*H]}_{fn}) \, ,
\end{equation*}
where $d_{\beta}$ is the $\beta$-divergence~\cite{Cichocki_A_2011_j-entropy_generalized_abdtarnmf,
  Fevotte_C_2011_j-neural-comput_algorithms_nmfbd} given by
\begin{equation*}
  d_{\beta}(x \mid y) =
  \begin{cases*}
    x\log \frac{x}{y} - x + y          & if $\beta=1$ \\
    \frac{x}{y} - \log \frac{x}{y} - 1 & if $\beta=0$ \\
    \frac{x^{\beta}}{\beta(\beta-1)}+\frac{y^{\beta}}{\beta} - \frac{x y^{\beta-1}}{\beta-1}
    & otherwise.
  \end{cases*}
\end{equation*}
The choice of $\beta$ can be made in accordance with the application or the assumed
noise model for $\*V$~\cite{Fevotte_C_2011_j-neural-comput_algorithms_nmfbd}.
Common values for $\beta$ are $0$, $1$, and $2$; they correspond to the IS divergence,
KL divergence and squared Frobenius norm, respectively.
The range $\beta \in [0,2]$ is the one with largest practical interest.
Values of $\beta \in [0, 0.5]$ are, for example, customary in audio spectral
decomposition~\cite{Fevotte_C_2009_j-neural-comput_nonnegative_mfisdama,
  Vincent_E_2010_j-ieee-trans-asl-proc_adaptive_hsdmpe}.
Values of $\beta \in [1,2]$ have proven efficient in hyperspectral
unmixing~\cite{Fevotte_C_2015_j-ieee-trans-img-proc_nonlinear_hurnmf}.
Note that the latter interval is important because the $\beta$-divergence $d_{\beta}(x|y)$
is convex with respect to (w.r.t.) $y$ when $\beta \in [1,2]$.
In that case, the optimization subproblems in $\*W$ and $\*H$ are separately
convex, though $\mJ(\*W,\*H)$ is always jointly non-convex.


\subsection{Necessity of the constrained formulation}
\label{ssec:well_pb}
As mentioned in Section~\ref{sec:intro}, removing the unit-norm constraint
in~\eqref{eq:pb1}, i.e., solving
\begin{equation}
  \label{eq:illposed_pb}
  \min_{\*W, \*H \ge 0 }  D_{\beta}(\*V \mid \*W\*H) + \alpha\norm{\*H}_{1} \, ,
\end{equation}
would lead to an ill-posed problem.
Indeed, $D_{\beta}$ suffers from a scaling ambiguity since it only depends on the
product $\*W\*H$ and not on $\*W$ and $\*H$ separately.
As a consequence, Problem~\eqref{eq:illposed_pb} is ill-posed: for any solution
$(\*W^{*},\*H^{*})$, there is always a solution $(\tau\*W^{*},\frac{1}{\tau}\*H^{*})$,
with $\tau>1$ a positive real constant, that yields a better minimizer, i.e.,
$\mJ(\tau\*W^{*},\frac{1}{\tau}\*H^{*}) < \mJ(\*W^{*},\*H^{*})$.
Hence, the function $\mJ$ is not coercive: for any feasible point $(\*W,\*H)$,
we can find a real number $\tau>1$ such that the sequence
${\left\{\tau^{k}\*W,\frac{1}{\tau^{k}}\*H\right\}}_{k\in\NN}$ remains feasible but diverges in
norm while the sequence ${\left\{\mJ(\tau^{k}\*W,\frac{1}{\tau^{k}}\*H)\right\}}_{k\in\NN}$ is
decreasing and bounded.
Therefore, the infimum of~\eqref{eq:illposed_pb} is never attained and there
exists no minimizer $(\boldsymbol{\*W^{*}},\boldsymbol{\*H^{*}})$.

Controlling the norm of $\*W$ is a natural way of tackling the ill-posedness
of~\eqref{eq:illposed_pb}.
This can be done by adding a penalizing term $\norm{\*W}$ (for a chosen norm) to
$\mJ(\*W,\*H)$~\cite{Kim_H_2007_j-bioinformatics_sparse_nmfanclsmda,
  Zhao_R_2018_j-ieee-trans-sig-proc_unified_camuarnmf}.
Alternatively, we may minimize $\mJ(\*W,\*H)$ subject to the additional
constraint $\norm{\*W} \le \theta$ where $\theta$ is a positive hyperparameter (it can be
easily shown that this returns solutions such that
$\norm{\*W} = \theta$)~\cite{Bolte_J_2013_j-math-prog_proximal_almnnp,
    Fu_X_2018_j-ieee-sig-proc-lett_identifiability_nmf}.
We chose in this paper to minimize $\mJ(\*W,\*H)$ subject to the additional
constraint that the individual columns of $\*W$ have unit norm, leading to
Problem~\eqref{eq:pb1}.
This is a rather natural option for dictionary learning, as it makes sense to
retrieve atoms (the columns of $\*W$) that have equal norm.
Instead, a constraint on the norm of the full matrix $\*W$ can return a solution
with atoms of different weights.
We chose to constrain the $\ell_{1}$ norm of the columns of
$\*W$ for practical commodity.
Section~\ref{sec:soa} presents the state-of-the-art methods for solving
Problem~\eqref{eq:pb1}~\cite{LeRoux_J_2015_TR_sparse_nmfhwd,
  Leplat_V_2021_j-siam-j-matrix-anal-appl_multiplicative_unmfbddec} while
Section~\ref{sec:snmf} introduces our block-descent MM algorithm.


\section{State of the art}
\label{sec:soa}


\subsection{Lagrangian method}
\label{ssec:lagr}

A standard method to solve optimization problems with constraints is the method
of Lagrange multipliers.
This method has been suggested by~\cite{Leplat_V_2021_j-siam-j-matrix-anal-appl_multiplicative_unmfbddec}
to solve $\beta$-NMF under a wide possibility of linear equality constraints on
either $\*W$ or $\*H$.
The Lagrangian associated to Problem~\eqref{eq:pb1} can be written as
\begin{multline}
  \label{eq:lagr}
  \Lagr(\*W,\*H,\bm{\nu}) \defequal \\
  D_{\beta}(\*V \mid \*W\*H)  + \alpha\norm{\*H}_{1}
  - \sum_{k=1}^{K}\nu_{k}(\norm{\*w_{k}}_{1}-1) \, ,
\end{multline}
where $\bm{\nu}={[\nu_{1},\ldots,\nu_{K}]}^{\top}\in\RR^{K}$ is the vector of Lagrangian
multipliers.

The saddle points of $\Lagr(\*W,\*H,\bm{\nu})$ subject to $\*W,\*H \ge 0$ yield
solutions of Problem~\eqref{eq:pb1}.
As such, the authors of~\cite{Leplat_V_2021_j-siam-j-matrix-anal-appl_multiplicative_unmfbddec}
describe a block-coordinate algorithm that alternately updates the blocks $\*W$,
$\*H$, and $\bm{\nu}$.
Given $\bm{\nu}$, the individual updates of $\*W$, $\*H$ are handled with one step
of block-descent MM, using the methodology
of~\cite{Fevotte_C_2011_j-neural-comput_algorithms_nmfbd,
  Yang_Z_2011_j-ieee-trans-nnet_unified_dmalqnmf} (see also
Section~\ref{ssec:mm_princip}).
\footnote{To be accurate,~\cite{Leplat_V_2021_j-siam-j-matrix-anal-appl_multiplicative_unmfbddec}
uses a slightly different formulation. Indeed, instead of optimizing the
Lagrangian~\eqref{eq:lagr} associated to Problem~\eqref{eq:pb1}, they optimize
the Lagrangian associated with the problem of minimizing a majorizer of
$C(\*W)=D_{\beta}(\*V|\*W\*H)$ subject to unit norm constraints.
The resulting updates turn out to be the same.}
This leads to the following updates when $\beta \le 1$
\begin{align}
  \*H &\; \longleftarrow \; \*H \odot {\left(\frac{\*W^{\top}\*S_{\beta}}{\*W^{\top}\*T_{\beta}+\alpha}\right)}^{.\frac{1}{2-\beta}}
  \label{eq:lagr_upH} \\
  \*W &\; \longleftarrow \; \*W \odot {\left(\frac{\*S_{\beta}\*H^{\top}}{\*T_{\beta}\*H^{\top}-\*1_{F}\bm{\nu}^{\top}}\right)}^{.\frac{1}{2-\beta}}
  \, ,
  \label{eq:lagr_upW}
\end{align}
where the matrices $\*S_{\beta}$ and $\*T_{\beta}$ are defined as
\begin{align}
  \*S_{\beta} &= \*V\odot{(\*W\*H)}^{.(\beta-2)}, \label{eq:shortcut1} \\
  \*T_{\beta} &= {(\*W\*H)}^{.(\beta-1)}. \label{eq:shortcut2}
\end{align}
Other closed-form updates can be obtained for $\beta \in \{\frac{5}{4}, \frac{4}{3}, \frac{3}{2}, 2\}$.
In every case, only the update of $\*W$ depends on $\bm{\nu}$, which we highlight
next with the abusive notation $\*W(\bm{\nu})$.
The $k$-th multiplier $\nu_{k}$ must ensure that $\*W(\bm{\nu})$ given
by~\eqref{eq:lagr_upW} satisfies $\norm{\*w_{k}(\nu_{k})}_{1}=1$, i.e.,
\begin{align*}
  \sum_{f} w_{fk}(\nu_{k}) = 1 \, .
\end{align*}
The latter equation involves finding the root of a rational function and has no
closed-form solution.
However, the authors of~\cite{Leplat_V_2021_j-siam-j-matrix-anal-appl_multiplicative_unmfbddec}
show that it has a unique solution that can be estimated with a standard
Newton-Raphson procedure in about 10 to 100 subiterations.
The solution is also shown to ensure that the denominator in~\eqref{eq:lagr_upW}
remains positive so that the update is well-defined and preserves nonnegativity.

Overall, the Lagrangian method~\cite{Leplat_V_2021_j-siam-j-matrix-anal-appl_multiplicative_unmfbddec}
is conceptually well-grounded and elegant.
It ensures that $\*W$ satisfies the desired norm constraint at every iteration
and also ensures non-increasingness of $\mJ(\*W,\*H)$.
However it applies to specific values of $\beta$ and requires a numerical subroutine
for the estimation of the Lagrange multipliers.


\subsection{Heuristic multiplicative updates}
\label{ssec:heuristic}

Another way to solve Problem~\eqref{eq:pb1} was proposed
in~\cite{Eggert_J_2004_p-ijcnn_sparse_cnmf} for NMF with the quadratic loss and
extended to $\beta$-NMF in~\cite{LeRoux_J_2015_TR_sparse_nmfhwd}.
It consists first in formulating~\eqref{eq:pb1} as an unconstrained
problem based on a reparametrization of the factor $\*W$.
More precisely, the factor $\*W$ is replaced by the normalized factor
$\*W \bm{\Lambda}^{-1}$, where $\bm{\Lambda}$ is a $K \times K$ diagonal matrix with entries
$\lambda_{k}=\norm{\*w_{k}}$ and $\norm{.}$ is some chosen norm.
The original papers~\cite{Eggert_J_2004_p-ijcnn_sparse_cnmf,
  LeRoux_J_2015_TR_sparse_nmfhwd} consider $\ell_{2}$ normalization.
We adapt their methodology to $\ell_{1}$ normalization for a fair comparison with the
other methods considered in this paper.
This results in the following minimization problem
\begin{equation}
  \label{eq:pb2}
  \begin{aligned}
    \min_{\*W, \*H \ge 0 } \quad \tilde{\mJ}(\*W,\*H)
    \defequal D_{\beta}\left(\*V \mid \*W \bm{\Lambda}^{-1} \*H\right) + \alpha\norm{\*H}_{1}
    \, .
  \end{aligned}
\end{equation}

The authors of~\cite{Eggert_J_2004_p-ijcnn_sparse_cnmf,
  LeRoux_J_2015_TR_sparse_nmfhwd} then propose to solve
Problem~\eqref{eq:pb2} using a block-alternating algorithm that updates $\*W$
and $\*H$ in turns.
Multiplicative updates for each factor are obtained by employing a heuristic
commonly used in NMF, see~\cite{Cichocki_A_2006_book_csiszar_dnmffna,
  Fevotte_C_2009_j-neural-comput_nonnegative_mfisdama}.
Looking at the update of $\*W$, the heuristic first consists in decomposing the
gradient $\nabla_{\*W}\tilde{\mJ}(\*W,\*H)$ of $\tilde{\mJ}(\*W,\*H)$ w.r.t. $\*W$
into the difference of two nonnegative functions, i.e.,
$\nabla_{\*W}\tilde{\mJ} = \nabla_{\*W}^{+}\tilde{\mJ} - \nabla_{\*W}^{-}\tilde{\mJ}$.
Such a decomposition does exist for the considered function, though it might not
exist for other problems.
Then, given a current iterate of $\*H$, a multiplicative update of $\*W$ is
constructed as 
\begin{align}
  \label{eq:multidea}
  \*W \; \longleftarrow \; \*W \odot \frac{\nabla_{\*W}^{-}\tilde{\mJ}(\*W,\*H)}{\nabla_{\*W}^{+}\tilde{\mJ}(\*W,\*H)}.
\end{align}
The motivating principle of update~\eqref{eq:multidea} is as follows.
Assume that ${[\nabla_{\*W}\tilde{\mJ}]}_{fk}>0$ for a given coefficient $w_{fk}$,
then the ratio in~\eqref{eq:multidea} is lower than $1$ and the multiplicative
update decreases $w_{fk}$ as it can be expected from a descent algorithm.
Likewise, update~\eqref{eq:multidea} increases the value of $w_{fk}$ when the
gradient is negative.
The heuristic works well in practice but does not come with any guarantee.
In particular, it does not ensure that $\tilde{\mJ}$ decreases at every
iteration (and it turns out that $\tilde{\mJ}$ sometimes increases).
Applying the heuristic to $\*W$ and $\*H$ in Problem~\eqref{eq:pb2} leads to the
following updates
\begin{align}
  \*H &\; \longleftarrow \; \*H \odot \frac{\*W^{\top}\*S_{\beta}}{\*W^{\top}\*T_{\beta}+\alpha} \label{eq:heur_upH} \\
  \*W &\; \longleftarrow \; \*W \odot \frac{\*S_{\beta}\*H^{\top}+\*1_{F \times F}(\*W\odot\*T_{\beta}\*H^{\top})}{\*T_{\beta}\*H^{\top}+\*1_{F \times F}(\*W\odot\*S_{\beta}\*H^{\top})} \label{eq:heur_upW} \\
  \*W &\; \longleftarrow \; \*W \bm{\Lambda}^{-1} \label{eq:renorm_w}
  \, ,
\end{align}
where $\*S_{\beta}$ and $\*T_{\beta}$ are defined in~\eqref{eq:shortcut1}
and~\eqref{eq:shortcut2}.

Note that the updates~\eqref{eq:lagr_upH} and~\eqref{eq:heur_upH} coincide up to
the exponent $1/(2-\beta)$.
As a matter of fact, when the other variables are fixed, the problems of
minimizing $\Lagr$ and $\tilde{\mJ}$ w.r.t. $\*H$ are identical, but solved
differently.
The MM update~\eqref{eq:lagr_upH} can be generalized to all values of $\beta$ as
shown in~\cite[Supplementary Material]{Tan_V_2013_j-ieee-trans-pami_automatic_rdnmfd},~\cite{Zhao_R_2018_j-ieee-trans-sig-proc_unified_camuarnmf}.
This means that we could use a theoretically-grounded MM update of $\*H$ instead
of the heuristic~\eqref{eq:heur_upH}.
However, omitting the exponent often results in a beneficial acceleration in
practice and we stick to the formulation of~\cite{LeRoux_J_2015_TR_sparse_nmfhwd}
given by~\eqref{eq:heur_upH} for fair comparison.

In summary, the approach proposed by~\cite{LeRoux_J_2015_TR_sparse_nmfhwd} is
intuitive, easy to implement, and applicable in principle for all values of $\beta$.
Unfortunately, it lacks theoretical support.
In the next section, we present a theoretically sound algorithm that results
in equally simple multiplicative updates with equal or better performance.


\section{A unified block-descent MM algorithm for $\beta$-NMF with $\ell_{1}$
  regularization}
\label{sec:snmf}

In this section, we first reformulate~\eqref{eq:pb1} into a well-posed
optimization problem that is free of norm constraints.
This is similar to the approach of~\cite{LeRoux_J_2015_TR_sparse_nmfhwd} except
that we use a different reformulation.
The reformulated problem allows to derive a block-descent MM algorithm that
results in simple multiplicative updates for both $\*W$ and $\*H$.
By design, the algorithm ensures that the objective function values are
non-increasing and convergent.
Additionally, we show that the sequence of iterates produced by the algorithm
also converges.


\subsection{Equivalent scale-invariant objective function}
\label{ssec:scale_inv_pb}


\subsubsection{Reformulation without norm constraints}

Let us introduce the following problem
\begin{equation}
  \label{eq:pb3}
  \min_{\*W, \*H \ge 0} \quad \check{\mJ}(\*W,\*H)
  \defequal D_{\beta}(\*V \mid \*W\*H) + \alpha\norm{\bm{\Lambda}\*H}_{1} \, ,
\end{equation}
where $\bm{\Lambda}$ is defined like in Section~\ref{ssec:heuristic}, i.e.,
$\bm{\Lambda} =\Diag\left(\norm{\*w_{1}}_{1},\ldots,\norm{\*w_{K}}_{1}\right)$.
Let us denote by $\FF$ the feasible set of Problem~\eqref{eq:pb1}, i.e.
\[
  \FF = \{(\*W,\*H)\in\RR_{+}^{F \times K}\times\RR_{+}^{K \times N} | (\forall k\in\nint{1,K}) \ \norm{\*w_{k}}_{1}=1\}
  \, .
\]
Then for any $(\*W,\*H) \in \FF$, we have $\check{\mJ}(\*W,\*H) = \mJ(\*W,\*H)$.
The following lemmas show that Problem~\eqref{eq:pb3} and Problem~\eqref{eq:pb1}
are equivalent.

\begin{lem}
  \label{lem:equiv}
  Let $(\*W^{*}, \*H^{*}) \ge 0$ be a solution of Problem~\eqref{eq:pb3}.
  Let us define their renormalized equivalents by $\bar{\*W}^{*} = \*W^{*} \bm{\Lambda}^{*-1}$
  and $\bar{\*H}^{*} = \bm{\Lambda}^{*} \*H^{*} $ where
  $\bm{\Lambda}^{*}=\Diag\left(\norm{\*w^{*}_{1}}_{1},\ldots,\norm{\*w^{*}_{K}}_{1}\right)$.
  Then, $(\bar{\*W}^{*},\bar{\*H}^{*})$ is a solution of Problem~\eqref{eq:pb1}.
\end{lem}

\begin{proof}
  Assume that $\bar{\*W}^{*}$, $\bar{\*H}^{*}$ is not a solution of
  Problem~\eqref{eq:pb1}.
  Then, there exists $(\bar{\*W}^{+}, \bar{\*H}^{+}) \in \FF$ such that
  $\mJ(\bar{\*W}^{+}$, $\bar{\*H}^{+}) < \mJ(\bar{\*W}^{*}$, $\bar{\*H}^{*})$.
  By design, we have $\mJ(\bar{\*W}^{*}, \bar{\*H}^{*}) = \check{\mJ}(\*W^{*},\*H^{*})$.
  Furthermore, ${\mJ}(\bar{\*W}^{+}$, $\bar{\*H}^{+}) =\check{\mJ}(\bar{\*W}^{+}$, $\bar{\*H}^{+})$.
  It follows that $\check{\mJ}(\bar{\*W}^{+},\bar{\*H}^{+}) < \check{\mJ}(\*W^{*},\*H^{*})$,
  which contradicts the assumption that $(\*W^{*}, \*H^{*})$ is a solution of
  Problem~\eqref{eq:pb3}. 
\end{proof}

\begin{lem} 
  Let $(\bar{\*W}^{*},\bar{\*H}^{*}) \in \FF$ be a solution of Problem~\eqref{eq:pb1}.
  Then $(\bar{\*W}^{*},\bar{\*H}^{*})$ is a solution of Problem~\eqref{eq:pb3}.
\end{lem}

\begin{proof}
  This follows from $\check{\mJ}(\bar{\*W},\bar{\*H}) = {\mJ}(\bar{\*W},\bar{\*H})$
  when $(\bar{\*W},\bar{\*H}) \in \FF$.
\end{proof}

Thanks to Lemma~\ref{lem:equiv}, we can solve Problem~\eqref{eq:pb3} without
norm constraints and renormalize the solution to obtain a solution to
Problem~\eqref{eq:pb1}.


\subsubsection{Symmetry of the roles of $\*W$ and $\*H$}

Note that the penalty term $\norm{\bm{\Lambda}\*H}_{1}$ in~\eqref{eq:pb3} now
depends on $\*W$.
Interestingly, it can be expanded and written as follows
\begin{align} \label{eq:l1pen}
  \norm{\bm{\Lambda}\*H}_{1} = \sum_{k,n} \norm{\*w_{k}}_{1} h_{k,n} = \sum_{f,k,n} w_{fk}h_{kn}
  = \sum_{f,k} \norm{\underline{\*h}_{k}}_{1} w_{fk},
\end{align}
where $\underline{\*h}_{k}$ denotes the $k^{th}$ row of $\*H$, and the indices
$f,k,n$ run from $1$ to $F,K,N$, respectively.
This shows that the updates of $\*H$ and $\*W$ in alternating minimization
correspond to equivalent problems: the roles of $\*H$ and $\*W$ can be exchanged
by transposition of $\*V$.
However, this property is specific to sparse NMF with $\ell_{1}$-regularization and
unit $\ell_{1}$-norm constraints.
The symmetry does not hold for example for the log-regularization considered
in Section~\ref{sec:log-snmf}.
Going further, our study shows that sparse NMF with $\ell_{1}$-regularization and
unit $\ell_{1}$-norm constraints is equivalent to decomposing the matrix $\*V$ into
sparse rank-1 matrices.
This is because~\eqref{eq:l1pen} can also be written as
$\norm{\bm{\Lambda}\*H}_{1} = \sum_{k} \norm{\*w_{k} \underline{\*h}_{k}}_{1}$, which induces
a mutual sparsity of $\*W$ and $\*H$.
Note that this is not a consequence of the reformulation~\eqref{eq:pb3},
but instead revealed by the equivalence between~\eqref{eq:pb3}
and~\eqref{eq:pb1}.
Next, we describe a block-descent MM algorithm for Problem~\eqref{eq:pb3}.
We start by recalling the principle of MM\@.


\subsection{Principle of majorization-minimization}
\label{ssec:mm_princip}

MM is a two-step iterative optimization method with a long history and renewed
interest, see recent overviews~\cite{Lange_K_2016_book_mm_oa,
  Sun_Y_2017_j-ieee-trans-sig-proc_majorization_maspcml}.
Let $C(\*X)$ be a real-valued function to minimize over its domain $\EE$,
where $\*X$ is a matrix variable of arbitrary size.
Let $\bXtilde \in \EE$ be a current iterate.
The majorization step of MM consists of building an {\em auxiliary function}
$G(\*X | \bXtilde)$ which is an upper bound of $C$ that is locally tight at
$\bXtilde$.
Mathematically, it must satisfy the two following properties
\begin{align}
  (\forall\*X\in\EE) \quad
  G(\*X \mid \bXtilde) &\geq C(\*X) \label{eq:mm_prop1} \\
  G(\bXtilde \mid \bXtilde) &= C(\bXtilde) \label{eq:mm_prop2} \, .
\end{align}
The minimization step of MM consists in minimizing \linebreak $G(\*X|\bXtilde)$
w.r.t. $\*X$, or at least finding an update $\hat{\*X}$ such that
$G(\hat{\*X} | \bXtilde) \le G(\bXtilde | \bXtilde)$.
This results in the following descent lemma
\begin{equation}
  \label{eq:desc_lemm}
  C(\hat{\*X}) \leq G(\hat{\*X} \mid \bXtilde) \leq G(\bXtilde \mid \bXtilde) = C(\bXtilde)
  \, .
\end{equation}
As such, MM ensures by design that the objective function $C$ is
non-increasing at every iteration.
Convergence of the iterates of $\*X$ is not straightforward and usually involves
problem-dependent assumptions, see, e.g.,~\cite{Zhao_R_2018_j-ieee-trans-sig-proc_unified_camuarnmf}
for NMF\@.
Next, we apply MM to alternating minimization of $\*W$ and $\*H$ for
Problem~\eqref{eq:pb3}.


\subsection{Construction of an auxiliary function for sparse NMF}
\label{ssec:aux_fct}

In this section we are interested in the minimization of the functions
$\*H \mapsto \check{\mJ}(\*W,\*H)$ (with fixed $\*W$) and $\*W \mapsto \check{\mJ}(\*W,\*H)$
(with fixed $\*H$).
As explained at the end of Section~\ref{ssec:scale_inv_pb}, these two optimization
problems are essentially the same and we will only address the first one.
Given $\*W$, our strategy to build an auxiliary function $G(\*H | \bHtilde)$ for
$C(\*H) = D_{\beta}(\*V | \*W\*H) + \alpha\norm{\bm{\Lambda}\*H}_{1}$ consists in majorizing the
data-fitting and regularization terms separately, and adding up the resulting
functions.


\subsubsection{Majorization of the data-fitting term}

Producing an auxiliary function for the data-fitting term $\*H \mapsto D_{\beta}(\*V | \*W\*H)$
is a well-known problem and we use the existing results
of~\cite{Nakano_M_2010_p-ieee-mlsp_convergence_gmanmfbd,
  Fevotte_C_2011_j-neural-comput_algorithms_nmfbd,
  Yang_Z_2011_j-ieee-trans-nnet_unified_dmalqnmf}.
For all values of $\beta$, the data-fitting term can be decomposed into the sum of a
convex function and a concave function.
The convex term may be majorized using Jensen's inequality while the concave
term may be majorized with the tangent inequality.
Adding up the two resulting functions results in the auxiliary function
$G_{\beta}(\*H | \bHtilde)$ given in Table~\ref{table:def_g1}.
This procedure has been used in many NMF papers,
including~\cite{Leplat_V_2021_j-siam-j-matrix-anal-appl_multiplicative_unmfbddec},
and the details can be found in~\cite{Fevotte_C_2011_j-neural-comput_algorithms_nmfbd}.

\begin{table}[t]
  \caption{Expression of the auxiliary function $G_{\beta}(\*H | \bHtilde)$
    for the data-fitting term $\*H \mapsto D_{\beta}(\*V|\*W\*H)$, up to an additive constant
    (from~\cite{Fevotte_C_2011_j-neural-comput_algorithms_nmfbd}).
    We use the following notations: $\tilde{p}_{kn}$ denotes the elements of the
    matrix $ \*W^{\top}\tilde{\*S}_{\beta}$, where $\tilde{\*S}_{\beta}$ is computed
    from~\eqref{eq:shortcut1} with $\*H = \bHtilde$.
    Similarly, $\tilde{q}_{kn}$ denotes the elements of $ \*W^{\top}\tilde{\*T}_{\beta}$.}
  \begin{center}
    \setlength\tabcolsep{5.5pt}
    \begin{tabular}{ll}
      \toprule
       & \multicolumn{1}{c}{$G_{\beta}\left(\*H|\bHtilde\right)$} \\
      \midrule
      $\beta<1$ & $\displaystyle \sum_{k,n} \left[\tilde{q}_{kn} h_{kn} - \frac{1}{\beta-1}\tilde{p}_{kn}\tilde{h}_{kn} {\left(\frac{h_{kn}}{\tilde{h}_{kn}}\right)}^{\beta-1}\right]$ \\
      $\beta=1$ & $\displaystyle \sum_{k,n} \left[\tilde{q}_{kn} h_{kn} - \tilde{p}_{kn} \tilde{h}_{kn} \log \left( \frac{h_{kn}}{\tilde{h}_{kn}}\right)\right]$ \\
      $\beta\in(1,2]$ & $\displaystyle \sum_{k,n} \left[\frac{1}{\beta}\tilde{q}_{kn}\tilde{h}_{kn}{\left(\frac{h_{kn}}{\tilde{h}_{kn}}\right)}^{\beta} - \frac{1}{\beta-1}\tilde{p}_{kn}\tilde{h}_{kn} {\left(\frac{h_{kn}}{\tilde{h}_{kn}} \right)}^{\beta-1}\right]$ \\
      $\beta>2 $ & $\displaystyle \sum_{k,n} \left[\frac{1}{\beta}\tilde{q}_{kn}\tilde{h}_{kn} {\left(\frac{h_{kn}}{\tilde{h}_{kn}}\right)}^{\beta} - \tilde{p}_{kn}h_{kn}\right]$ \\
      \bottomrule
      \end{tabular}
  \end{center}
  \label{table:def_g1}
\end{table}


\subsubsection{Majorization of the regularization term}

We now address the majorization of $S(\*H) = \norm{\bm{\Lambda}\*H}_{1} = \sum_{k,n} \lambda_{k} h_{kn}$,
where we recall that $\lambda_{k} = \norm{\*w_{k}}_{1}$.
We need to distinguish two cases: when $\beta \le 1$, no majorization of $S$ is
actually needed.
Indeed, in that case we may use
\begin{align}
  G(\*H \mid \bHtilde) = G_{\beta}(\*H \mid \bHtilde) + \alpha \, S(\*H) \, ,
\end{align}
because $G(\*H | \bHtilde)$ has a simple closed-form minimizer, given
in Section~\ref{ssec:mini}.
When $\beta>1$, this property is no longer true.
In that case, we need to majorize $S(\*H)$ as well.
Following~\cite{Yang_Z_2011_j-ieee-trans-nnet_unified_dmalqnmf}, we use the
following inequality that holds for $h, \tilde{h}>0$ and $\beta>1$
\begin{align}
  \label{eq:monome}
  h \le \frac{\tilde{h}}{\beta} {\left(\frac{h}{\tilde{h}}\right)}^{\beta}
  + \tilde{h}\left(1-\frac{1}{\beta}\right)
  \, .
\end{align}
Note that the inequality is tight when $h = \tilde{h}$.
Applied term to term to $S(\*H)$, this leads to the following majorizer of the
regularization term
\begin{align*}
  G_{S}(\*H \mid \bHtilde) =
  \sum_{k,n} \lambda_{k} \frac{\tilde{h}_{kn}}{\beta}{\left(\frac{h_{kn}}{\tilde{h}_{kn}}\right)}^{\beta}
  + \text{cst}
  \, ,
\end{align*}
where $\text{cst}$ contains terms that are constant w.r.t. $h_{kn}$.
\footnote{We will use the same notation $\text{cst}$ in different places to
avoid cluttering, though the constants might be different.}
In the end, when $\beta > 1$, we use
\begin{align*}
  G(\*H \mid \bHtilde) = G_{\beta}(\*H \mid \bHtilde) + \alpha \, G_{S}(\*H \mid \bHtilde) \, ,
\end{align*}
which admits a simple closed-form solution given in the next section.
The extra majorization step~\eqref{eq:monome} essentially allows
$G(\*H | \bHtilde)$ to be composed of monomials of only two different orders,
$\beta$ and $\beta-1$, hence allowing for closed-form minimization.


\subsection{Minimization of the auxiliary function}
\label{ssec:mini}

The second step of MM consists of minimizing $G(\*H | \bHtilde)$ w.r.t. $\*H$.
By design, $G$ is smooth, separable and strictly convex and thus we only need to
set its gradient to zero (w.r.t. $\*H$).
This step involves standard calculus and leads in the end to the following
update
\begin{align*}
  h_{kn} =
  \tilde{h}_{kn} {\left(\frac{\tilde{p}_{kn}}{\tilde{q}_{kn} + \alpha\norm{\*w_{k}}_{1}}\right)}^{\gamma(\beta)}
  \, ,
\end{align*}
where $\tilde{p}_{kn}$ and $\tilde{q}_{kn}$ are defined in
Table~\ref{table:def_g1} and
\begin{equation*}
  \gamma(\beta) =
  \begin{cases*}
    \frac{1}{2-\beta} & if $\beta<1$     \\
                1 & if $\beta\in[1,2]$ \\
    \frac{1}{\beta-1} & if $\beta>2$ 
  \end{cases*} \, .
\end{equation*}
As explained before, a similar update can be derived for $w_{fk}$ by exchanging
the roles of $\*W$ and $\*H$.
In the end, this leads to the following multiplicative matrix updates
\begin{align}
  \*H \; &\longleftarrow \; \*H \odot {\left(\frac{\*W^{\top}\*S_{\beta}}{\*W^{\top}(\*T_{\beta}+\alpha\*1_{F \times N})}\right)}^{.\gamma(\beta)} \label{eq:mm-nmf_H} \\
  \*W \; &\longleftarrow \; \*W \odot {\left(\frac{\*S_{\beta}\*H^{\top}}{(\*T_{\beta}+\alpha\*1_{F \times N})\*H^{\top}}\right)}^{.\gamma(\beta)} \label{eq:mm-nmf_W}
  \, .
\end{align}
A pseudo-code of the resulting procedure, coined MM-SNMF-$\ell_{1}$, is given in
Algorithm~\ref{algo:mm-nmf}.
A few comments are in order.
First, as in standard NMF practice, only one step of MM is applied to $\*H$ and
$\*W$ in each iteration.
Applying several sub-iterations brings no benefit in practice.
Like $\mJ$, $\tilde{\mJ}$ or $\Lagr$, the objective function $\check{\mJ}$ is
non-convex w.r.t. $\*W$ and $\*H$.
When $\beta \in [1,2]$, the individual sub-problems in $\*W$ and $\*H$ are convex,
but $\check{\mJ}$ is still jointly non-convex.
As such, initialization matters and we will present average performance results
over several random initializations in Section~\ref{sec:simul}.
Several data-dependent initialization schemes are presented
in~\cite{Gillis_N_2020_book_nonnegative_mf}.
In the next section we discuss the convergence of the iterates produced
by Algorithm~\ref{algo:mm-nmf}.

\begin{algorithm}[t]
  \small
  \begin{algorithmic}[1]
    \renewcommand{\algorithmicrequire}{\textbf{Input:}}
    \renewcommand{\algorithmicensure}{\textbf{Output:}}
    \REQUIRE{Nonnegative matrix $\*V$, initialization
    $(\*W_{\text{init}},\*H_{\text{init}})$,\\ and $\alpha>0$.}
    \ENSURE{Nonnegative matrices $\*W$ and $\*H$ such that $\*V \approx \*W\*H$ with
      sparse $\*H$.}
    \STATE{Initialize $i$ to $0$.}
    \STATE{Initialize $(\*W_{i},\*H_{i})$ to $(\*W_{\text{init}},\*H_{\text{init}})$.}
    \REPEAT{}
    \STATE{Update $\*H_{i}$ using~\eqref{eq:mm-nmf_H}:
      \[
      \begin{aligned}
        \bVtilde &\leftarrow \*W_{i}\*H_{i} \\
        \*H_{i+1} &\leftarrow
        \*H_{i}\odot{\left(\frac{\*W_{i}^{\top}\left(\*V\odot\bVtilde^{.(\beta-2)}\right)}{\*W_{i}^{\top}(\bVtilde^{.(\beta-1)}+\alpha\*1_{F \times N})}\right)}^{.\gamma(\beta)}
      \end{aligned}
      \]}
    \STATE{Update $\*W_{i}$ using~\eqref{eq:mm-nmf_W}:
      \[
      \begin{aligned}
        \bVtilde &\leftarrow \*W_{i}\*H_{i+1} \\
        \*W_{i+1} &\leftarrow
        \*W_{i} \odot {\left(\frac{\left(\*V\odot\bVtilde^{.(\beta-2)}\right)\*H_{i+1}^{\top}}{(\bVtilde^{.(\beta-1)}+\alpha\*1_{F \times N})\*H_{i+1}^{\top}}\right)}^{.\gamma(\beta)}
      \end{aligned}
    \]}
    \STATE{Increment $i$.}
    \UNTIL{stopping criterion is met}
    \STATE{Rescale $\*W_{i}$ and $\*H_{i}$:
      \[
      \begin{aligned}
      \bm{\Lambda} &\leftarrow \Diag\left({(\norm{\*w_{k}}_{1})}_{k\in\nint{1,K}}\right) \\
      (\*W_{i},\*H_{i}) &\leftarrow (\*W_{i}\bm{\Lambda}^{-1},\bm{\Lambda}\*H_{i})
      \end{aligned}
      \]}
    \RETURN{$(\*W_{i},\*H_{i})$}
  \end{algorithmic}
  \caption{MM-SNMF-$\ell_{1}$~\label{algo:mm-nmf}}
\end{algorithm}


\subsection{Convergence}

By construction, the sequence of objective values produced by
Algorithm~\ref{algo:mm-nmf} is non-increasing.
Because $\check{\mJ}$ is bounded below by zero, the sequence thus converges.
The following theorem additionally states the convergence of the iterates.

\begin{theorem}
  \label{thm:cvg_it}
  For any data matrix $\*V\in\RR_{+}^{F \times N}$, rank $K\in\NN$ and regularization
  parameter $\alpha>0$, the sequence of iterates $\{\*W_{i}, \*H_{i}\}_{i\in\NN}$ generated by
  Algorithm~\ref{algo:mm-nmf} converges to the set of stationary points of
  Problem~\eqref{eq:pb3}.
  \footnote{Due to the coercivity and the continuity of $\check{\mJ}$, the
  sequence $\{\*W_{i}, \*H_{i}\}_{i\in\NN}$ has at least one limit point.
  As such, the convergence of $\{\*W_{i}, \*H_{i}\}_{i\in\NN}$ to the set of stationary
  points means that every limit point of $\{\*W_{i}, \*H_{i}\}_{i\in\NN}$ is a stationary
  point.
  A stationary point is defined as a feasible point for which the necessary
  optimality condition given by Euler's inequality holds.}
\end{theorem}

\begin{proof}
  The authors in~\cite{Zhao_R_2018_j-ieee-trans-sig-proc_unified_camuarnmf} prove
  the convergence of the iterates of a block-descent MM algorithm constructed
  like Algorithm~\ref{algo:mm-nmf} for the following problem
  \begin{equation*}
    \min_{\*W,\*H \ge 0} \quad D_{\beta}(\*V \mid \*W\*H) + \alpha_{1}\norm{\*H}_{1} + \alpha_{2}\norm{\*W}_{1}
    \, .
  \end{equation*}
  As a matter of fact, their proof of convergence can be applied step by step to
  our own block-descent MM approach for solving Problem~\eqref{eq:pb3}.
  Indeed, the auxiliary functions that we derived for $\*H$ and equivalently for
  $\*W$ satisfy the five properties required
  in~\cite[Definition 2]{Zhao_R_2018_j-ieee-trans-sig-proc_unified_camuarnmf} to
  establish convergence.
  Using $G(\*H | \bHtilde)$ for exposition, the five properties are the
  following.
  \begin{itemize}
  \item Property 1 and 2 correspond to Equations~\eqref{eq:mm_prop1}
    and~\eqref{eq:mm_prop2} that define a valid auxiliary function.
  \item Property 3 dictates that $\nabla_{h_{kn}} G(\*H | \bHtilde)$ is a function of
    $h_{kn}/\tilde{h}_{kn}$.
  \item Property 4 dictates that the directional derivatives of
    $G(\*H | \bHtilde)$ and $\check{\mJ}(\*H)$ coincide at $\*H = \tilde{\*H}$.
  \item Property 5 dictates that $G(\*H | \bHtilde)$ is strictly convex w.r.t.
    $\*H$.
  \end{itemize}
  
  Standard calculus and convex analysis show that Properties 3--5 are satisfied
  by $G(\*H | \bHtilde)$, whereas Properties 1--2 are satisfied by construction.
  The results of~\cite{Zhao_R_2018_j-ieee-trans-sig-proc_unified_camuarnmf} also
  require $\check{\mJ}$ to be coercive which is satisfied thanks to the
  regularization term $\norm{\bm{\Lambda}\*H}_{1}$.
\end{proof}


\section{Extension to NMF with the $\beta$-divergence and log-regularization}
\label{sec:log-snmf}

In this section, we extend the methodology of Section~\ref{ssec:heuristic} and
Section~\ref{sec:snmf} to sparse $\beta$-NMF with log-regularization.
More precisely, we are interested in solving
\begin{align}
  \label{eq:wellposed_log_pb}
  \min_{\*W,\*H \ge 0} {\mJ}_{\log}(\*W,\*H)
  \quad \textrm{s.t.} \ (\forall k\in \nint{1,K}) \norm{\*w_{k}}_{1}=1
  \, ,
\end{align}
where
\begin{align}
  \label{eq:Jlog}
  {\mJ}_{\log}(\*W,\*H) \defequal D_{\beta}(\*V\mid\*W\*H) + \alpha \sum_{k,n} \log(h_{kn} + \epsilon)
  \, .
\end{align}
The log-regularization term $\psi(x) = \log(\abs{x} + \epsilon)$ used in~\eqref{eq:Jlog}
was popularized by~\cite{Candes_E_2008_j-four-anal-appl_enhancing_srlm} for sparse
linear regression ($x \in \RR$).
For small positive $\epsilon$, this function is much sharper at the origin than the
$\ell_{1}$ norm.
As such, it accentuates the sparsity of the solutions, which can be necessary or
desired in practice.
In the context of NMF, it was considered in~\cite{Lefevre_A_2011_p-icassp_itakura_isnmfgs,
  Tan_V_2013_j-ieee-trans-pami_automatic_rdnmfd,
  Peng_C_2022_j-knowledge-bases-sys_log_snmfdr}.
Following the discussion in Section~\ref{ssec:well_pb}, the unit-norm
constraints in~\eqref{eq:wellposed_log_pb} ensure that the minimization problem
is well-posed.

Changing the regularization term in $\mJ(\*W,\*H)$ only influences the update of
$\*H$ in the Lagrangian and heuristic methods described in Sections~\ref{ssec:lagr}
and~\ref{ssec:heuristic}.
Indeed, the update of $\*W$ is unchanged given $\*H$.
This is not true with our approach described in Section~\ref{sec:snmf} because
the regularization term in the reformulated scale-invariant objective function
depends on both $\*W$ and $\*H$.
Furthermore, under the log-regularization term, the minimization problems w.r.t.
$\*W$ and $\*H$ are not exchangeable anymore but can still be handled in the MM
framework.

In Section~\ref{ssec:log-heuristic}, we first extend the method
of~\cite{LeRoux_J_2015_TR_sparse_nmfhwd} to Problem~\eqref{eq:wellposed_log_pb}
and derive heuristic multiplicative updates that appear to work in practice.
Then, we derive our principled block-descent MM algorithm in
Section~\ref{ssec:log-MM}.
The Lagrangian method from~\cite{Leplat_V_2021_j-siam-j-matrix-anal-appl_multiplicative_unmfbddec}
could be also extended to Problem~\eqref{eq:wellposed_log_pb} for values of $\beta \leq 1$
by combining the update of $\*W$ in Section~\ref{ssec:lagr} with our update of
$\*H$ in Section~\ref{ssec:log-MM}.
However this does not change the conclusions of Section~\ref{ssec:lagr} about
the limitations of the Lagrangian approach and we chose to omit this method in
our experimental comparisons when considering log-regularization.


\subsection{Heuristic multiplicative updates}
\label{ssec:log-heuristic}

We adapt the approach of~\cite{LeRoux_J_2015_TR_sparse_nmfhwd} by replacing $\*W$
with $\*W \bm{\Lambda}^{-1} $ in~\eqref{eq:Jlog} and using alternating multiplicative
updates of $\*W$ and $\*H$ derived using the heuristic~\eqref{eq:multidea}.
By standard calculus, this results in the following updates
\begin{align}
  \label{eq:log_heur_up1}
  \*H &\; \longleftarrow \; \*H \odot \frac{\*W^{\top}\*S_{\beta}}{\*W^{\top}\*T_{\beta}+\frac{\alpha}{\*H+\epsilon}} \\
  \*W &\; \longleftarrow \; \*W \odot \frac{\*S_{\beta}\*H^{\top}+\*1_{F \times F}(\*W\odot\*T_{\beta}\*H^{\top})}{\*T_{\beta}\*H^{\top}+\*1_{F \times F}(\*W\odot\*S_{\beta}\*H^{\top})}
  \label{eq:log_heur_up2} \\
  \*W &\; \longleftarrow \; \*W \bm{\Lambda}^{-1}
  \, ,
\end{align}
where $\*S_{\beta}$ and $\*T_{\beta}$ are defined in~\eqref{eq:shortcut1}
and~\eqref{eq:shortcut2}.
As stated before, only the update of $\*H$ is changed when compared to the
updates derived in Section~\ref{ssec:heuristic} for $\beta$-NMF with $\ell_{1}$
regularization.


\subsection{Block-descent majorization-minimization algorithm}
\label{ssec:log-MM}

We now apply the methodology of Section~\ref{sec:snmf} to
Problem~\eqref{eq:wellposed_log_pb}.
Following Section~\ref{ssec:scale_inv_pb} we can show that
Problem~\eqref{eq:wellposed_log_pb} is equivalent to
\begin{equation}
  \label{eq:rescaled_log_eq}
  \min_{\*W, \*H \ge 0 } \quad \check{\mJ}_{\log} \defequal 
  D_{\beta}(\*V \mid \*W\*H) + \alpha\sum_{k,n}\psi\left(\lambda_{k} h_{k,n}\right) \, ,
\end{equation}
where we recall that $\lambda_{k} = \norm{\*w_{k}}_{1}$.
As opposed to $\check{\mJ}$, the roles of $\*W$ and $\*H$ are not exchangeable
anymore in $\check{\mJ}_{\log}$ and we now proceed to derive separate MM updates
for the two factors.


\subsubsection{Update of $\*H$}

Given $\*W$, we start by constructing an auxiliary function $G(\*H|\bHtilde)$
for the function $C(\*H) = D_{\beta}(\*V | \*W\*H) + \alpha S(\*H)$, where
$S(\*H) = \sum_{k,n}\psi\left(\lambda_{k} h_{k,n}\right)$.
We use the same notations $C$, $S$ and $G$ as in Section~\ref{sec:snmf} in order
to avoid cluttering.
We majorize the data-fitting term with the same function $G_{\beta}(\*H | \bHtilde)$
than before, given in Table~\ref{table:def_g1}.
We now turn to the majorization of $S(\*H)$.

By concavity of the logarithm, the individual summands of $S(\*H)$ can be
majorized locally at $\bHtilde$ with the tangent inequality
\begin{align}
  \psi(\lambda_{k} h_{kn})
  \leq \psi(\lambda_{k} \tilde{h}_{kn})
  + \lambda_{k} \psi'(\lambda_{k}\tilde{h}_{k}) (h_{kn}-\tilde{h}_{kn}) \, ,
  \label{eq:tang}
\end{align}
where $\psi'(x) = 1/(x+\epsilon)$ for all $x \in \RR_{+}$.
From there, we need to distinguish two cases like in Section~\ref{ssec:aux_fct}.

When $\beta \le 1$, we may simply apply~\eqref{eq:tang} to and use the following
auxiliary function for $S(\*H)$
\begin{align*}
  G_{S}(\*H|\bHtilde)
  = \sum_{k,n} \frac{h_{kn}}{\tilde{h}_{kn} + \frac{\epsilon}{\lambda_{k}}} + \text{cst} \, ,
\end{align*}
where $\text{cst}$ contains terms that are constant w.r.t. $h_{kn}$.
This leads to an auxiliary function
$G(\*H|\bHtilde) = G_{\beta}(\*H|\bHtilde) + \alpha G_{S}(\*H|\bHtilde)$ that has a simple
closed-form minimizer when $\beta \le 1$.
The minimization is infeasible when $\beta >1 $ and we need to resort to an
additional majorization step, using again~\eqref{eq:monome}.
This leads to the following auxiliary function for $S(\*H)$
\begin{align*}
  G_{S}(\*H,\bHtilde)
  = \frac{1}{\beta} \sum_{k,n} \frac{\tilde{h}_{kn}}{\tilde{h}_{kn}+\frac{\epsilon}{\lambda_{k}}}
  {\left(\frac{h_{kn}}{\tilde{h}_{kn}}\right)}^{\beta}
  + \text{cst}
  \, .
\end{align*}
In the end, for all $\beta \in \RR$, $G$ is a smooth, separable and strictly convex
function that is easily minimized by setting its gradient to zero.
This leads to the following multiplicative update
\begin{equation}
  \label{eq:log_mm-nmf_h}
    \*H \; \longleftarrow \;
    \*H \odot {\left(\frac{\*W^{\top}\*S_{\beta}}{\*W^{\top}\*T_{\beta}+\frac{\alpha}{\*H+\frac{\epsilon}{\bm{\Upsilon}}}}\right)}^{.\gamma(\beta)}
    \, ,
\end{equation}
where $\bm{\Upsilon} = \*W^{\top}\*1_{F \times N}$.


\subsubsection{Update of $\*W$}
    
A very similar strategy can be employed for the update of $\*W$. Given $\*H$, we
now need to build an auxiliary function $F(\*W|\bWtilde)$ for the minimization
of $B(\*W) = D_{\beta}(\*V | \*W\*H) + \alpha R(\*W) $, where
$R(\*W) = \sum_{k,n} \psi(h_{kn}\norm{\*w_{k}}_{1})$.
The data-fitting term can be majorized by switching the roles of $\*W$ and $\*H$
in Table~\ref{table:def_g1}; slightly abusing the notations again we denote the
resulting auxiliary function by $G_{\beta}(\*W | \bWtilde)$.
Let us now address the majorization of $R(\*W)$.
Given the current update $\bWtilde$, we may again invoke the concavity of $\psi(x)$
to form the following inequality
\begin{align}
  \psi(h_{kn} \norm{\*w_{k}}_{1})
    & \leq  \psi(h_{kn}\norm{\bwtilde_{k}}_{1})
    + \frac{\norm{\*w_{k}}_{1}-\norm{\bwtilde_{k}}_{1}}{\norm{\bwtilde_{k}}_{1}+\frac{\epsilon}{h_{kn}}} \nonumber \\
    &= \frac{1}{\norm{\bwtilde_{k}}_{1}+\frac{\epsilon}{h_{kn}}} \sum_{f} w_{fk}  + \text{cst}
    \, .
    \label{eq:ineq2}
\end{align}
From there, we use a path that is similar to the update of $\*H$.
When $\beta \le 1$, we may simply apply inequality~\eqref{eq:ineq2} to the summands of
$R(\*W)$ and use the following majorizer
\begin{align*}
  F_{R}(\*W| \bWtilde)
  = \sum_{f,k} \left(\sum_{n}\frac{1}{\norm{\bwtilde_{k}}_{1}+\frac{\epsilon}{h_{kn}}}\right)w_{fk}
  + \textrm{cst}
  \, .
\end{align*}
When $\beta >1 $ we need to further majorize the terms $w_{fk}$
using~\eqref{eq:monome}.
In the end, the minimization of
$F(\*W| \bWtilde) = G_{\beta}(\*W | \bWtilde) + \alpha F_{R}(\*W | \bWtilde)$ leads to the
following multiplicative update of $\*W$
\begin{equation}
  \label{eq:log_mm-nmf_w}
    \*W \; \longleftarrow \;
    \*W \odot {\left(\frac{\*S_{\beta}\*H^{\top}}{\*T_{\beta}\*H^{\top}+ \*1_{F \times N}{\left(\frac{\alpha}{\bm{\Upsilon}+\frac{\epsilon}{\*H}}\right)}^{\top}}\right)}^{.\gamma(\beta)}
  \, .
\end{equation}


\subsubsection{Resulting algorithm and convergence}

Our resulting MM algorithm to address~\eqref{eq:rescaled_log_eq} is displayed in
Algorithm~\ref{algo:log_mm-nmf} and is referred to as MM-SNMF-log\@.
By design, the algorithm ensures that the sequence of objective values
${(\check{\mJ}_{\log}(\*W_{i},\*H_{i}))}_{i\in\NN}$ is non-increasing and convergent.
The convergence of the iterates can also be proven, using the same rationale as
the proof of Theorem~\ref{thm:cvg_it}.
We simply need $G$ and $F$ to verify the five properties listed in the proof,
which is easily checked.
Properties 1--3 hold by construction of the auxiliary functions, Property 4 can
be verified using standard calculus, and Property 5 results from convexity
properties.
In the end, we have derived the first universal algorithm for sparse $\beta$-NMF
with log-regularization.
The algorithm is simple to implement, has linear complexity per iteration and
can be applied for any value of $\beta \in \RR$.
It is free of tuning parameters and enjoys strong convergence properties.

\begin{algorithm}[t]
  \small
  \begin{algorithmic}[1]
    \renewcommand{\algorithmicrequire}{\textbf{Input:}}
    \renewcommand{\algorithmicensure}{\textbf{Output:}}
    \REQUIRE{Nonnegative matrix $\*V$,initialization
    $(\*W_{\text{init}},\*H_{\text{init}})$,\\ and $\alpha>0$.}
    \ENSURE{Nonnegative matrices $\*W$ and $\*H$ such that $\*V \approx \*W\*H$ with
      sparse $\*H$.}
    \STATE{Initialize $i$ to $0$.}
    \STATE{Initialize $(\*W_{i},\*H_{i})$ to $(\*W_{\text{init}},\*H_{\text{init}})$.}
    \REPEAT{}
    \STATE{$\bm{\Upsilon} \leftarrow \*W_{i}^{\top}\*1_{F \times N}$}
    \STATE{Update $\*H_{i}$ using~\eqref{eq:log_mm-nmf_h}:
      \[
      \begin{aligned}
        \bVtilde &\leftarrow \*W_{i}\*H_{i} \\
        \*H_{i+1} &\leftarrow
        \*H_{i}\odot{\left(\frac{\*W_{i}^{\top}\left(\*V\odot\bVtilde^{.(\beta-2)}\right)}{\*W_{i}^{\top}\bVtilde^{.(\beta-1)}+\frac{\alpha}{\*H_{i}+\frac{\epsilon}{\bm{\Upsilon}}}}\right)}^{.\gamma(\beta)}
      \end{aligned}
      \]}
    \STATE{Update $\*W_{i}$ using~\eqref{eq:log_mm-nmf_w}:
      \[
      \begin{aligned}
        \bVtilde &\leftarrow \*W_{i}\*H_{i+1} \\ 
        \*W_{i+1} &\leftarrow
        \*W_{i} \odot {\left(\frac{\left(\*V\odot\bVtilde^{.(\beta-2)}\right)\*H_{i+1}^{\top}}{ \bVtilde^{.(\beta-1)}\*H_{i+1}^{\top}+\*1{\left(\frac{\alpha}{\bm{\Upsilon}+\frac{\epsilon}{\*H_{i+1}}}\right)}^{\top}}\right)}^{.\gamma(\beta)}
      \end{aligned}
    \]}
    \STATE{Increment $i$.}
    \UNTIL{stopping criterion is met}
    \STATE{Rescale $\*W_{i}$ and $\*H_{i}$:
      \[
      \begin{aligned}
      \bm{\Lambda} &\leftarrow \Diag\left({(\norm{\*w_{k}}_{1})}_{k\in\nint{1,K}}\right) \\
      (\*W_{i},\*H_{i}) &\leftarrow (\*W_{i}\bm{\Lambda}^{-1},\bm{\Lambda}\*H_{i})
      \end{aligned}
      \]}
    \RETURN{$(\*W_{i},\*H_{i})$}
  \end{algorithmic}
  \caption{MM-SNMF-log~\label{algo:log_mm-nmf}}
\end{algorithm}


\section{Experimental Results}
\label{sec:simul}

In this section, we compare our MM methods against the Lagrangian and the
heuristic methods presented in Section~\ref{sec:soa} on four different datasets.
We first give an example showing that the heuristic is not a descent algorithm
in contrast with our MM method.
We then compare MM-SNMF-$\ell_{1}$ described by Algorithm~\ref{algo:mm-nmf} for solving
Problem~\eqref{eq:pb1} with its Lagrangian and heuristic counterparts presented in
Sections~\ref{ssec:lagr} and~\ref{ssec:heuristic}.
We refer to the latter two methods as L-SNMF-$\ell_{1}$ and H-SNMF-$\ell_{1}$ respectively.
We finally compare our algorithm MM-SNMF-log described by Algorithm~\ref{algo:log_mm-nmf}
with H-SNMF-log which is the variant of the heuristic method given in
Section~\ref{ssec:log-heuristic} and aimed at solving
Problem~\eqref{eq:wellposed_log_pb}.


\subsection{Description of the datasets and hyperparameter choices}

To compare the algorithms under realistic conditions, we select four different
datasets coming from various applications that are described below.
\begin{itemize}
\item The Olivetti dataset from AT\&T Laboratories
  Cambridge~\cite{Samaria_F_1994_p-wacv_parametrisation_smhfi} contains $400$
  greyscale images of faces with dimensions $64 \times 64$ that are vectorized and
  stored as the columns of $\*V$.
  From these images, NMF can be used to learn part-based features that are
  represented by the dictionary of features $\*W$~\cite{Lee_D_1999_j-nature_learning_ponmf}.
  The factor $\*H$ then contains the activation encodings of the features for
  each image of the collection.
\item We generate an audio magnitude spectrogram from an excerpt of the
  original recording of the song ``Four on Six'' by Wes Montgomery.
  The signal corresponds to the first five seconds of the song sampled at
  \SI{44.1}{\kilo\hertz}.
  The spectrogram is then computed with a Hamming window of length $1024$ (23ms)
  and with an overlap of $50\%$.
  The use of NMF in this context consists in extracting elementary audio
  time-frequency patterns represented in $\*W$ with their temporal activations
  given by $\*H$~\cite{Smaragdis_P_2014_j-ieee-sig-proc-mag_static_dssunmfuv}.
\item The Moffett dataset is a hyperspectral image with resolution $50 \times 50$
  pixels over $189$ spectral bands acquired over Moffett Field in 1997 by the
  Airborne Visible Infrared Imaging Spectrometer~\cite{Aviris_database}.
  Using NMF on such an image allows extracting a dictionary $\*W$ of individual
  spectra representing the different materials, as well as their relative
  proportions stored in
  $\*H$~\cite{BioucasDias_J_2012_j-ieee-j-sel-top-appl-earth-obs-rem-sens_hyperspectral_uogssrba}.
\item The TasteProfile
  dataset~\cite{Bertin-Mahieux_T_2011_book_million_sd} contains counts of songs
  played by users of a music streaming service.
  In this context, NMF may extract the user preferences represented by the
  matrix $\*W$ as well as the different song attributes represented by the
  matrix $\*H$~\cite{Hu_Y_2008_p-ieee-icdm_collaborative_fifd}.
  We apply a preprocessing to the dataset similarly to~\cite{Gouvert_O_2020_p-icml_ordinal_nmdr}
  and many other papers using this dataset: we keep only users and songs with
  more than twenty interactions.
  The latter preprocessing still results in a large and highly sparse dataset.
\end{itemize}
Table~\ref{table:dataset} displays the dimensions of each dataset together with
the values of $\beta$ and $\alpha$ that we have used in our experiments.
The value of $\alpha_{\ell_{1}}$ is used for Problem~\eqref{eq:pb1} while
the value of $\alpha_{\log}$ is used for Problem~\eqref{eq:wellposed_log_pb}.
The constant $\epsilon$ in the log-regularization is set to $0.01$. 
For the spectrogram and Moffett dataset, we perform tests with two different
values of $\beta$ and thus use different values of $\alpha$ accordingly.
Note that we have tested several values of the regularization parameter $\alpha$ and
we have chosen one that yields representative results.
In particular, for the chosen values, the regularization does not become
negligible in comparison with the data term and conversely.

The values of $\beta$ have been chosen according to standard practice, see,
e.g.,~\cite{Fevotte_C_2011_j-neural-comput_algorithms_nmfbd} and references in
Section~\ref{sec:obj}.
Values of $\beta$ in the $[0, 0.5]$ interval are recommended for audio spectra as
they give more importance to small-energy coefficients.
The value $\beta=1$ produces a data-fitting term that corresponds to the
log-likelihood of a Poisson model; this fits well with integer-valued data such
as counts (TasteProfile) or RGB images (Olivetti).
Values of $\beta$ in the $[1, 2]$ interval offer a good compromise between Poisson
and additive Gaussian noise assumptions, which suits well to hyperspectral data
(Moffett).

\begin{table}[t]
  \caption{Dimensions of the datasets used in our experiments}
    \begin{center}
      \setlength\tabcolsep{3pt}
      \begin{tabular}{lcccccc}
        \toprule
        & $F$ & $N$ & $K$ & $\beta$ & $\alpha_{\ell_{1}}$ & $\alpha_{\log}$ \\
        \midrule
        Olivetti     &  $4\,096$ &     $400$ & $10$ &      $1$ &      $0.01$ &          $5$ \\
        Spectrogram  &     $513$ &     $858$ & $10$ & $\{0,0.5\}$ &   $\{600,5\}$ &    $\{0.5,5\}$ \\
        TasteProfile & $16\,301$ & $12\,118$ & $50$ &      $1$ &      $5000$ &         $0.5$ \\
        Moffett      &     $189$ &  $2\,500$ &  $3$ & $\{1.3,2\}$ & $\{1000,0.05\}$ & $\{0.5,0.02\}$ \\
        \bottomrule
      \end{tabular}
    \end{center}
    \label{table:dataset}
\end{table}


\subsection{Set-up}

All the simulations presented in this section have been conducted in Matlab 2021a
running on an Intel i7-8650U CPU with a clock cycle of 1.90GHz shipped with 16GB
of memory.\footnote{Matlab code is available at
\href{https://arthurmarmin.github.io/research.html}{https://arthurmarmin.github.io/research.html}.}

For each dataset, we compare the factorization obtained by the different
methods from $50$ different initializations.
The elements of $(\*W_{\text{init}},\*H_{\text{init}})$ are drawn randomly
according to a half-normal distribution obtained by folding a centered Gaussian
distribution of standard deviation equal to $5$.


\subsubsection{Stopping criterion}

The following stopping criterion has been used for all the algorithms 
\begin{equation}
  \label{eq:stop_crit}
  \frac{\abs{\mJ(\*W^{-},\*H^{-})-\mJ(\*W,\*H)}}{\abs{\mJ(\*W,\*H)}}
  \leq \delta \; ,
\end{equation}
where $\delta$ is a tolerance set to $10^{-5}$, $\*W$ and $\*H$ are the current
iterates while $\*W^{-}$ and $\*H^{-}$ are the previous ones.
The regularization term in $\mJ$ depends on the context and is either
the $\ell_{1}$ or the log-regularization.
The absolute value at the denominator is necessary in the case of the
log-regularization since the logarithm function, and thus $\mJ$, could be
negative.
If the convergence is not reached after $5,000$ iterations, we stop the
algorithm and return the current estimated factor matrices.


\subsubsection{Implementation}

The pseudo-code for MM-SNMF-$\ell_{1}$ and MM-SNMF-log is shown in
Algorithms~\ref{algo:mm-nmf} and~\ref{algo:log_mm-nmf} respectively.
The implementation of H-SNMF-$\ell_{1}$ and H-SNMF-log is similar except that the
multiplicative rules are replaced by the ones given
by~\eqref{eq:heur_upH},~\eqref{eq:heur_upW}
and~\eqref{eq:log_heur_up1},~\eqref{eq:log_heur_up2} respectively, together with
the renormalization~\eqref{eq:renorm_w}.

The implementation of L-SNMF-$\ell_{1}$ follows the one of MM-SNMF-$\ell_{1}$ but uses the
multiplicative updates given by~\eqref{eq:lagr_upH} and~\eqref{eq:lagr_upW}
instead of~\eqref{eq:mm-nmf_H} and~\eqref{eq:mm-nmf_W}.
Furthermore, an additional step consisting in computing the optimal Lagrangian
multipliers has to be performed between the steps 4 and 5 of
Algorithm~\ref{algo:mm-nmf}.

The $\beta$-divergence $d_{\beta}(x|y)$ is not always well defined when $x$ or $y$ takes
the value zero due to the presence of quotients and logarithms.
Consequently, we use in practice $D_{\beta}(\*V+\kappa|\*W\*H+\kappa)$ with a small constant $\kappa$
instead of the objective function $D_{\beta}(\*V|\*W\*H)$ for numerical stability.
For the heuristic method, this leads to replacing $\*W\*H$ by $\*W\*H + \kappa$ in
the expression of the gradient and thus in the updates~\eqref{eq:heur_upH}
and~\eqref{eq:heur_upW}.
For our method and the Lagrangian one (which both rely on MM), we can also
safely replace $\*W\*H$ by $\*W\*H + \kappa$ in the multiplicative updates.
This can be proven by treating $\kappa$ as a ${(K + 1)}^{\text{th}}$ constant
component in the derivations like
in~\cite{Fevotte_C_2015_j-ieee-trans-img-proc_nonlinear_hurnmf}.


\subsubsection{Performance evaluation}

We compare the different algorithms with two metrics: their computational
efficiency (CPU time) and the quality of the returned solutions.
The latter is assessed by the value of the normalized objective function
$\mJ(\*W,\*H)/FN$ at the solution returned by the algorithms.


\subsection{Results}


\subsubsection{Descent property}
\label{sssec:descent-prop}

We illustrate in this section that H-SNMF-$\ell_{1}$ is not a descent algorithm unlike
MM-SNMF-$\ell_{1}$\@.
To this end, we generate a random data matrix $\*V$ of dimension $50 \times 40$ by
drawing its elements according to the same half-normal distribution used for
drawing the elements of $(\*W_{\text{init}},\*H_{\text{init}})$.
We apply both H-SNMF-$\ell_{1}$ and MM-SNMF-$\ell_{1}$ with $K=3$, $\beta=-0.5$, and $\alpha=5$.
Then, we plot the values of the normalized objective function for the first
fifty iterations in Figure~\ref{fig:comp_crit}.
We use the same initialization for both methods but do not plot the value of the
objection function at the initialization (iteration 0) for the sake of clarity.
We observe that the blue curve corresponding to MM-SNMF-$\ell_{1}$ is non-increasing
whereas the red curve representing H-SNMF-$\ell_{1}$ is oscillating.
This example demonstrates the theoretical advantage to use MM-SNMF-$\ell_{1}$ over
H-SNMF-$\ell_{1}$\@.
Furthermore, we observe on this example that H-SNMF-$\ell_{1}$ requires more
iterations than MM-SNMF-$\ell_{1}$ to reach a given value of the objective function
close to a local optimum.
This observation will be verified in the experiments on the datasets in the
next section.

\begin{figure}[!t]
  \centering
  \includegraphics[width=\linewidth]{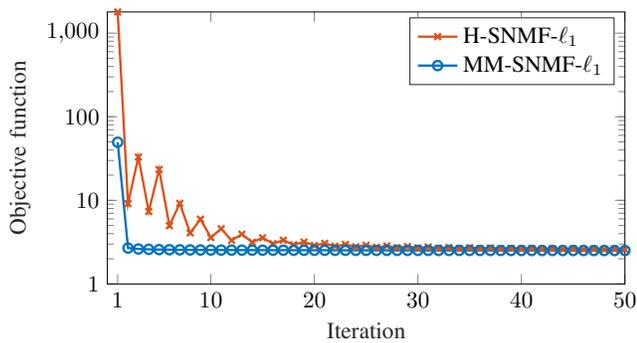}
  \caption{Values of the normalized objective function through the first hundred
    of iterations. Results obtained on synthetic data matrix $\*V$ with
    parameters $(F,N)=(50,40)$, $K=3$, $\beta=-0.5$.}
  \label{fig:comp_crit}
\end{figure}


\subsubsection{Performance comparison for Problem~\eqref{eq:pb1}}

We now run H-SNMF-$\ell_{1}$, L-SNMF-$\ell_{1}$, and MM-SNMF-$\ell_{1}$ on the four datasets.
The average values of the normalized objective function $\mJ/FN$ at the
solutions returned by the three algorithms are given in the top part of
Table~\ref{table:stats-l1}.
We observe first that optimal values of the objective function do not vary much
with the initialization for the three methods.
Moreover, we notice that H-SNMF-$\ell_{1}$ and MM-SNMF-$\ell_{1}$ yield solutions with a
similar quality while L-SNMF-$\ell_{1}$ may return higher-quality solutions.

On Figures~\ref{fig:oliv-taste_res},~\ref{fig:spectro_res},
and~\ref{fig:hyperspec_res}, we show the CPU time used by each method before
convergence.
For the sake of clarity, we only show the first $25$ realizations, the behaviour
of the others is similar.
We observe that the efficiency of the methods in term of CPU time depends on the
dataset, on the value of $\beta$ and on the initialization (e.g., for Moffett
or the spectrogram datasets).
However, we notice some general trends: for example, H-NMF-$\ell_{1}$ is always the
slowest on Olivetti and TasteProfile datasets.
The middle part of Table~\ref{table:stats-l1} exposes this trend by displaying
the corresponding average run times together with their standard deviations
shown within parentheses.
We observe that H-SNMF-$\ell_{1}$ is the slowest in average for every dataset while
MM-SNMF-$\ell_{1}$ is the fastest one except for TasteProfile for which L-SNMF-$\ell_{1}$ is
13$\%$ faster.

The three algorithms do not follow the same path in the optimization space.
In particular, we can observe in the bottom part of Table~\ref{table:stats-l1}
that MM-SNMF-$\ell_{1}$ converges in fewer iterations than H-SNMF-$\ell_{1}$\@.
Since the complexity per iteration of these methods are similar---compare the
multiplicative updates~\eqref{eq:heur_upH} with~\eqref{eq:mm-nmf_H}
and~\eqref{eq:heur_upW} with~\eqref{eq:mm-nmf_W}---this explains the observed
difference in CPU time.
Furthermore, one can see that L-SNMF-$\ell_{1}$ converges in a smaller number of
iterations than MM-SNMF-$\ell_{1}$ for some datasets.
However, its iterations are more expensive due to the update of the Lagrangian
multipliers through a Newton-Raphson iterative method.

\begin{figure}[!t]
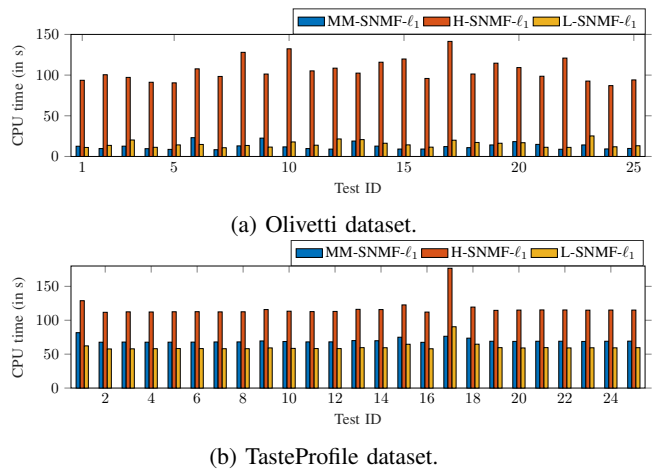

  \centering
  \begin{subfigure}[b]{\columnwidth}
    \includegraphics[width=\linewidth]{oliv_res.tex}
    \caption{Olivetti dataset.}
  \end{subfigure}
  \begin{subfigure}[b]{\columnwidth}
    \includegraphics[width=\linewidth]{tasteprof_res.tex}
    \caption{TasteProfile dataset.}
  \end{subfigure}     
  \caption{Comparative performance with Olivetti and TasteProfile datasets using
    the $\ell_{1}$-regularization ($\beta=1$).}
  \label{fig:oliv-taste_res}
\end{figure}

\begin{figure}[!t]
  \centering
  \begin{subfigure}[b]{\columnwidth}
    \includegraphics[width=\linewidth]{spectro_res.tex}
    \caption{$\beta=0$.}
  \end{subfigure}
  \begin{subfigure}[b]{\columnwidth}
    \includegraphics[width=\linewidth]{spectro-beta=05_res.tex}
    \caption{$\beta=0.5$.}
  \end{subfigure}     
  \caption{Comparative performance with a spectrogram using the
    $\ell_{1}$-regularization.}
  \label{fig:spectro_res}
\end{figure}

\begin{figure}[!t]
  \centering
  \begin{subfigure}[b]{\columnwidth}
    \includegraphics[width=\linewidth]{hyperspec-beta=13_res.tex}
    \caption{$\beta=1.3$.}
  \end{subfigure}
  \begin{subfigure}[b]{\columnwidth}
    \includegraphics[width=\linewidth]{hyperspec-beta=2_res.tex}
    \caption{$\beta=2$.}
  \end{subfigure}     
  \caption{Comparative performance with Moffett dataset using the
    $\ell_{1}$-regularization.}
  \label{fig:hyperspec_res}
\end{figure}

\begin{table}[t]
  \caption{Statistics for the three algorithms designed to solve Problem~\eqref{eq:pb1}.
    The top section shows the average values of the objective function $\mJ/FN$
    at the returned solutions.
    The middle section gives the average CPU times with the lowest ones
    highlighted in bold.
    The bottom section yields the corresponding average number of iterations.
    Standard deviations are given within parentheses.}
  \begin{center}
    \setlength\tabcolsep{2.5pt}
    \begin{tabular}{lccc}
      \toprule
      & L-SNMF-$\ell_{1}$ & H-SNMF-$\ell_{1}$ & MM-SNMF-$\ell_{1}$ \\
      \midrule
      &\multicolumn{3}{c}{\textbf{Objective function}} \\
      Olivetti              & 3.16   ($\pm$7E-3) &  3.16  ($\pm$9E-3) & 3.16   ($\pm$6E-3) \\
      Spectrogram ($\beta=0$)   & 0.88   ($\pm$3E-2) & 20.3   ($\pm$3E-2) & 20.3   ($\pm$3E-2) \\
      Spectrogram ($\beta=0.5$) & 0.60   ($\pm$3E-2) &  2.98  ($\pm$7E-3) & 2.98   ($\pm$6E-3) \\
      TasteProfile          & 0.76   ($\pm$7E-6) &  9.15  ($\pm$5E-6) & 9.15   ($\pm$5E-6) \\
      Moffett ($\beta=1.3$)     & ---              &  0.17  ($\pm$2E-5) & 0.17   ($\pm$2E-4) \\
      Moffett ($\beta=2$)       & 4.1E-3 ($\pm$7E-3) & 4.6E-3 ($\pm$1E-2) & 4.6E-3 ($\pm$7E-3) \\
      \midrule
      &\multicolumn{3}{c}{\textbf{CPU time}} \\
      Olivetti              &         14.1s  ($\pm$ 3.8) & 103.4s ($\pm$16.8) & \bfemph{11.2s} ($\pm$3.5) \\
      Spectrogram ($\beta=0$)   &          6.5s  ($\pm$ 3.4) &   6.4s ($\pm$ 2.2) & \bfemph{ 6.3s} ($\pm$1.8) \\
      Spectrogram ($\beta=0.5$) &         23.5s  ($\pm$ 2.5) &  24.5s ($\pm$ 6.3) & \bfemph{22.5s} ($\pm$7.2) \\
      TasteProfile          & \bfemph{60.6s} ($\pm$ 6.5) & 117.5s ($\pm$12.9) &         69.8s  ($\pm$3.4) \\
      Moffett ($\beta=1.3$)     &            ---           &   1.2s ($\pm$ 0.6) & \bfemph{ 0.9s} ($\pm$0.4) \\
      Moffett ($\beta=2$)       &          0.8s  ($\pm$ 0.3) &   1.1s ($\pm$ 0.3) & \bfemph{ 0.6s} ($\pm$0.1) \\
      \midrule
      &\multicolumn{3}{c}{\textbf{Number of iterations}} \\
      Olivetti              & 763 ($\pm$112) & 947 ($\pm$ 99) & 767 ($\pm$111) \\
      Spectrogram ($\beta=0$)   & 219 ($\pm$109) & 160 ($\pm$ 46) & 239 ($\pm$ 55) \\
      Spectrogram ($\beta=0.5$) & 197 ($\pm$ 93) & 144 ($\pm$ 37) & 183 ($\pm$ 55) \\
      TasteProfile          &  14 ($\pm$  0) &  23 ($\pm$  0) &  23 ($\pm$  0) \\
      Moffett ($\beta=1.3$)     & ---          &  12 ($\pm$  6) &  11 ($\pm$  6) \\
      Moffett ($\beta=2$)       & 137 ($\pm$ 25) & 133 ($\pm$ 49) &  95 ($\pm$ 19) \\
      \bottomrule
    \end{tabular}
  \end{center}
  \label{table:stats-l1}
\end{table}


\subsubsection{Performance comparison for Problem~\eqref{eq:wellposed_log_pb}}

We now compare H-SNMF-log with MM-SNMF-log.
Similarly to the previous section, the statistics on the values of the objective
function, on the CPU times, and on the numbers of iterations are shown in
Table~\ref{table:stats-log} while Figures~\ref{fig:log-oliv-taste_res},~\ref{fig:log-spectro_res},
and~\ref{fig:log-hyperspec_res} display the CPU time for the first $25$
Monte-Carlo realizations.
Results similar to the $\ell_{1}$-regularization can be observed: both methods return
solutions of nearly same quality whereas MM-SNMF-$\ell_{1}$ is significantly faster
for all datasets except for the spectrogram, for which both methods use in
average the same CPU time.
This difference in CPU time is explained by the number of iterations before
convergence as shown in Table~\ref{table:stats-log}: MM-SNMF-$\ell_{1}$ takes on
average about $20\%$ to $50\%$ less iterations than H-SNMF-$\ell_{1}$ to converge.
The difference in CPU time is particularly significant for the large scale
dataset TasteProfile.

\begin{figure}[!t]
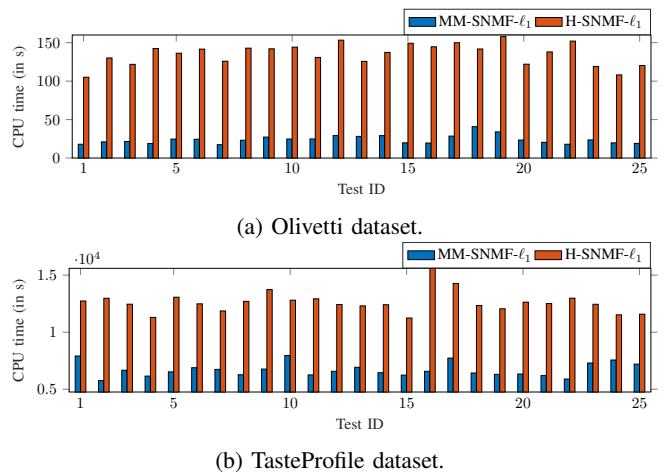

  \centering
  \begin{subfigure}[b]{\columnwidth}
    \includegraphics[width=\linewidth]{log-oliv_res.tex}
    \caption{Olivetti dataset.}
  \end{subfigure}
  \begin{subfigure}[b]{\columnwidth}
    \includegraphics[width=\linewidth]{log-tasteprof_res.tex}
    \caption{TasteProfile dataset.}
  \end{subfigure}     
  \caption{Comparative performance with Olivetti and TasteProfile datasets using
    the log-regularization ($\beta=1$).}
  \label{fig:log-oliv-taste_res}
\end{figure}

\begin{figure}[!t]
  \centering
  \begin{subfigure}[b]{\columnwidth}
    \includegraphics[width=\linewidth]{log-spectro_res.tex}
    \caption{$\beta=0$.}
  \end{subfigure}
  \begin{subfigure}[b]{\columnwidth}
    \includegraphics[width=\linewidth]{log-spectro-beta=05_res.tex}
    \caption{$\beta=0.5$.}
  \end{subfigure}
  \caption{Comparative performance with a spectrogram using the
    log-regularization.}
  \label{fig:log-spectro_res}
\end{figure}

\begin{figure}[!t]
  \centering
  \begin{subfigure}[b]{\columnwidth}
    \includegraphics[width=\linewidth]{log-hyperspec-beta=13_res.tex}
    \caption{$\beta=1.3$.}
  \end{subfigure}
  \begin{subfigure}[b]{\columnwidth}
    \includegraphics[width=\linewidth]{log-hyperspec-beta=2_res.tex}
    \caption{$\beta=2$.}
  \end{subfigure}     
  \caption{Comparative performance with Moffett dataset using the
    log-regularization.}
  \label{fig:log-hyperspec_res}
\end{figure}

\begin{table}[t]
  \caption{Statistics for the three algorithms designed to solve
    Problem~\eqref{eq:wellposed_log_pb}.
    The top section shows the average values of the objective function $\mJ/FN$
    at the returned solutions.
    The middle section gives the average CPU times with the lowest ones
    highlighted in bold.
    The bottom section yields the corresponding average number of iterations.
    Standard deviations are given within parentheses.
    Average numbers of iterations to solve Problem~\eqref{eq:wellposed_log_pb}.}
  \begin{center}
    \setlength\tabcolsep{5.5pt}
    \begin{tabular}{lcc}
      \toprule
      & H-SNMF-log & MM-SNMF-log \\
      \midrule
      &\multicolumn{2}{c}{\textbf{Objective function}} \\
      Olivetti              &  1.96    ($\pm$9E-3) &  1.96    ($\pm$7E-3) \\
      Spectrogram ($\beta=0$)   &  5.16E-1 ($\pm$4E-3) &  5.09E-1 ($\pm$4E-3) \\
      Spectrogram ($\beta=0.5$) & -6.58E-3 ($\pm$7E-2) & -5.68E-2 ($\pm$9E-3) \\
      TasteProfile          &  1.77E-2 ($\pm$4E-5) &  1.76E-2 ($\pm$2E-5) \\
      Moffett ($\beta=1.3$)     &  1.13E-3 ($\pm$6E-5) &  1.08E-2 ($\pm$5E-5) \\
      Moffett ($\beta=2$)       &  1.03E-3 ($\pm$1E-5) &  1.02E-2 ($\pm$1E-5) \\
      \midrule
      &\multicolumn{2}{c}{\textbf{CPU time}} \\
      Olivetti              &        132s  ($\pm$ 15) &   \bfemph{23s} ($\pm$  6) \\
      Spectrogram ($\beta=0$)   & \bfemph{17s} ($\pm$  4) &   \bfemph{17s} ($\pm$  4) \\
      Spectrogram ($\beta=0.5$) & \bfemph{34s} ($\pm$  6) &          39s  ($\pm$  9) \\
      TasteProfile          &      12613s  ($\pm$927) & \bfemph{6704s} ($\pm$603) \\
      Moffett ($\beta=1.3$)     &         25s  ($\pm$  4) &  \bfemph{ 16s} ($\pm$  2) \\
      Moffett ($\beta=2$)       &         10s  ($\pm$  2) &    \bfemph{7s} ($\pm$  3) \\
      \midrule
      &\multicolumn{2}{c}{\textbf{Number of iterations}} \\
      Olivetti              & 1180 ($\pm$ 130) & 920 ($\pm$112) \\
      Spectrogram ($\beta=0$)   &  512 ($\pm$  86) & 611 ($\pm$132) \\
      Spectrogram ($\beta=0.5$) &  198 ($\pm$  33) & 300 ($\pm$ 73) \\
      TasteProfile          & 1900 ($\pm$ 131) & 929 ($\pm$ 83) \\
      Moffett ($\beta=1.3$)     & 1001 ($\pm$ 240) & 801 ($\pm$206) \\
      Moffett ($\beta=2$)       & 1235 ($\pm$ 282) & 959 ($\pm$286) \\
      \bottomrule
    \end{tabular}
  \end{center}
  \label{table:stats-log}
\end{table}


\section{Conclusion}
\label{sec:concl}

We have presented a block-descent MM algorithm for $\beta$-NMF with
$\ell_{1}$-regularization or log-regularization on one factor and unit $\ell_{1}$-norm
constraint on the columns of the other.
Our algorithm takes the form of iterative multiplicative updates with are simple
and efficient to compute.
In contrast with state-of-the-art methods, our resulting algorithm can be
applied to every $\beta$-divergence and owns desirable theoretical properties
such as non-increasingness, convergence of the objective function as well as
the convergence of its iterates to the set of stationary points of the problem.
Furthermore, we have observed experimentally that our MM algorithm estimates
factors with competitive quality and leads in many cases to a significant
decrease of CPU time when compared to state-of-the-art methods.


\bibliographystyle{IEEEtran}
\bibliography{abbr,mybiblio}


\end{document}